%% file: Nonvacuous-DNN.tex
\documentclass{article}

\usepackage[utf8]{inputenc} \usepackage[T1]{fontenc}    \usepackage{hyperref}       \usepackage{url}            \usepackage{booktabs}       \usepackage{amsfonts}       \usepackage{nicefrac}       \usepackage{microtype}      
\usepackage[margin=1.125in]{geometry}
\usepackage[round]{natbib}
\usepackage{wrapfig}
\usepackage[normalem]{ulem} \usepackage[makeroom]{cancel} \usepackage[position=b]{subcaption}
\usepackage{graphicx}
\usepackage{amsmath,amsfonts,amssymb,amsthm,commath,dsfont}
\usepackage{enumitem}
\usepackage{xspace}
\usepackage{cleveref}
\usepackage[dvipsnames]{xcolor}
\usepackage{thmtools}
\makeatletter
\newcommand{\opnorm}{\@ifstar\@opnorms\@opnorm}
\newcommand{\@opnorms}[1]{  \left|\mkern-1.5mu\left|\mkern-1.5mu\left|
   #1
  \right|\mkern-1.5mu\right|\mkern-1.5mu\right|
}
\newcommand{\@opnorm}[2][]{  \mathopen{#1|\mkern-1.5mu#1|\mkern-1.5mu#1|}
  #2
  \mathclose{#1|\mkern-1.5mu#1|\mkern-1.5mu#1|}
}
\makeatother

\usepackage{mathtools}
\usepackage{dylan}

\input{math_commands.tex}

\usepackage{enumitem}
\setlist{nosep}
\usepackage{lscape}
\usepackage{tabularx}

\usepackage{graphicx}
\usepackage{caption}
\usepackage{subcaption}
\usepackage{multirow,xcolor}
\usepackage{wrapfig}
\usepackage{arydshln}

\theoremstyle{plain}
\newtheorem{theorem}{Theorem}

\newtheorem{lemma}[theorem]{Lemma}
\newtheorem{corollary}{Corollary}

\newtheorem{remark}{Remark}

\title{Non-vacuous Generalization Bounds for Deep Neural Networks \\ without  any modification to the trained models}

\author{
 Khoat Than\thanks{{\tt<\url{khoattq@soict.hust.edu.vn}>}; Hanoi University of Science and Technology, Hanoi, Vietnam.} 
 \and Dat Phan\thanks{VinBigdata Institute, Hanoi, Vietnam.}
}

\date{} 
\begin{document}

\maketitle

\begin{abstract}
Understanding and certifying the behavior of modern deep neural networks remains a fundamental challenge in reliable machine learning. We introduce a new class of data-dependent generalization bounds that apply directly to trained models, without any modification. In particular, we present an exactly computable bound that is non-vacuous across all evaluated networks, including ImageNet-scale models with 600M parameters. This this is the first work showing that meaningful generalization guarantees are achievable even for large, unaltered deep networks.
  
  Our approach reveals that generalization is governed by the interaction between the trained model and the geometry of the data distribution. We decompose the generalization error into two interpretable components: a distributional complexity term, capturing how the data mass is distributed across the input space, and local model-behavior terms, capturing the network’s behavior within individual regions. This joint dependence identifies where and why generalization gaps arise. Empirically, some components of our bound are highly predictive of the true test error, and the bound tightens when the partition aligns with the intrinsic data geometry, highlighting data-dependent local regularity as a key driver of  generalization. 
\end{abstract}

\section{Introduction}

Deep neural networks (NNs) have enabled remarkable advances across domains ranging from game playing to structural biology and large-scale language modeling \citep{silver2016mastering,jumper2021AlphaFold,achiam2023GPT4}. A striking aspect of these systems is their ability to generalize: models trained on finite datasets often achieve strong performance on previously unseen data. Understanding the origin of this generalization capability remains a central challenge in modern learning theory.

Despite decades of progress, classical frameworks struggle to explain this phenomenon in the regime of overparameterized, high-capacity networks. Approaches based on Rademacher complexity \citep{bartlett2017SpectralMarginDNN}, algorithmic stability \citep{bousquet2002stability,brutzkus2021optimization}, robustness \citep{xu2012robustnessGeneralize,sokolic2017robustDNN}, and PAC-Bayes analysis \citep{mcallester2003pac,biggs2022nonvacuousBound} offer valuable perspectives, but typically yield vacuous guarantees when applied to modern NN  architectures.

Recent work has begun to close this gap by demonstrating that non-vacuous bounds are attainable under carefully controlled settings. For instance, \citet{dziugaite2017computingBounds} obtained non-vacuous PAC-Bayes bounds via optimized posterior distributions, while \citet{zhou2019CompressionBound} leveraged compression to control the generalization error of stochastic networks. Stability-based approaches have also shown promise \citep{nadjahi24slicingMI}. However, these advances largely remain confined to small-scale models or rely on stochasticity and model modifications.

A notable breakthrough toward large-scale analysis was achieved by \citet{lotfi24NonVacuousLLM,lotfi24unlockingLLM}, who derived non-vacuous bounds for large language models such as GPT-2 and LLaMA2 using a combination of quantization, fine-tuning, and compression techniques. While these results significantly expand the scope of theoretical guarantees, they apply to modified versions of the original networks rather than to the trained models themselves. Consequently, it remains unclear whether such guarantees faithfully capture the generalization behavior of the unaltered model. This limitation highlights a fundamental open problem:

\textbf{Can we obtain a non-vacuous bound on the true error of a  trained NN, without altering it?}

Addressing this question is intrinsically challenging. The true error is the primary quantity governing generalization \citep{mohri2018foundations}, yet most existing techniques \citep{dziugaite2017computingBounds,zhou2019CompressionBound} either introduce stochasticity or rely on significant model transformations. Even recent PAC-Bayes approaches \citep{lotfi24NonVacuousLLM,lotfi24unlockingLLM} that reduce explicit randomness still depend on aggressive compression and fine-tuning, \textit{yielding guarantees for surrogate models}. Since the relationship between the original model $\vh$ and its modified counterpart $\vh'$ is generally uncontrolled, tight bounds for $\vh'$ need not translate into meaningful guarantees for $\vh$.

\begin{wrapfigure}{ri}{0.45\linewidth}
  \centering
  \includegraphics[width=\linewidth]{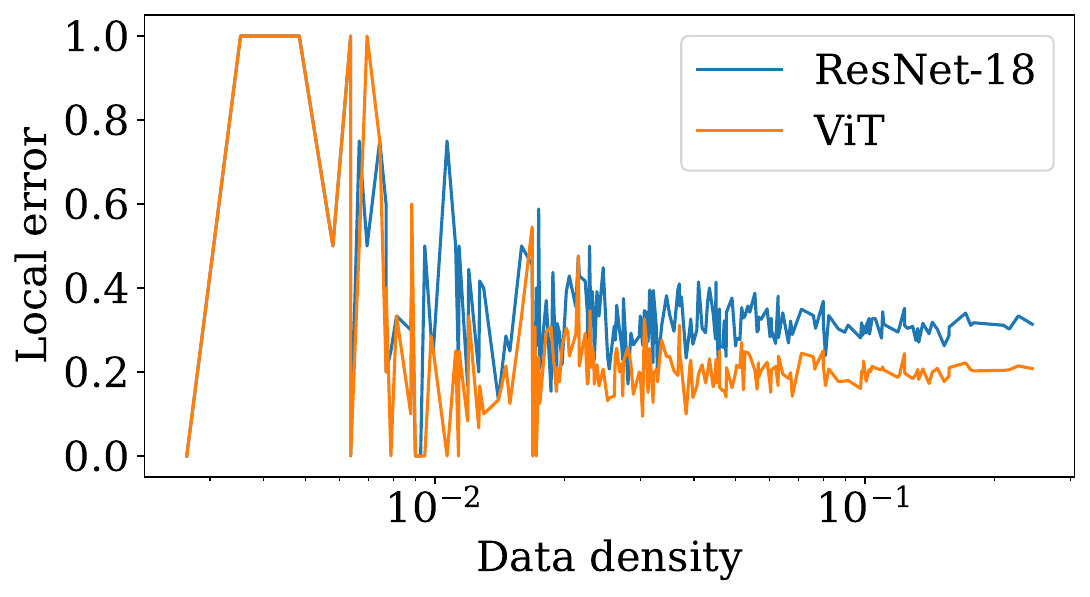}
      \caption{The behaviors of two ImageNet-scale trained models. The $x$-axis shows the data density  at each local region of the input space, while the $y$-axis reports  local errors.}
      \label{fig-density-error}
\end{wrapfigure}

At a deeper level, these difficulties stem from the reliance of existing theories on \emph{global} measures of complexity, whether of the hypothesis class, parameter space, or model itself. In contrast, modern NNs exhibit highly heterogeneous behavior across the input space: prediction errors tend to concentrate in specific regions as illustrated in Figure~\ref{fig-density-error}, and both the data distribution and the model’s local behavior critically influence generalization. Capturing this interaction requires moving beyond global worst-case analyses toward a more refined, spatially localized perspective.

In this work, we develop a \emph{geometry-aware} framework that decomposes the true error of a trained model into contributions from different regions of the input space. This perspective makes explicit how distributional mass and local model behavior jointly determine generalization. Crucially, our approach yields bounds that are both \emph{non-vacuous} and \emph{structurally informative}: they not only certify the true error, but also identify where errors arise and how they relate to the geometry of the data.

Our main contributions are as follows:
\begin{itemize}
\item \textit{First}, we derive a novel bound on the true error of a trained model $\vh$ under mild assumptions, based on a partition of the input space. The bound simultaneously captures distributional complexity and local properties of $\vh$, revealing how alignment between local errors and data geometry governs overall performance. Empirically, such an alignment highly correlates with true error, and hence can be reliably used for model selection and understanding generalization.

A key technical challenge is the control of an intractable aggregation term over local regions; we resolve this via a fine-grained analysis of small and binomial random variables.

\item \textit{Second}, we introduce a tractable variant that can be computed exactly from the training data (optionally, with few more held-out samples), without modifying $\vh$. This yields a practical certification procedure that retains the structural insights of the general theory. A detailed comparison with prior approaches is provided in Table~\ref{tab-Recent-approaches}.

\item \textit{Third}, we empirically validate our framework on 30 modern neural networks trained on ImageNet (1.2M samples), spanning CNNs and Transformers with up to 600M parameters. Across all models, we obtain non-vacuous upper bounds on the true error, even under conservative parameter choices. With suitable partitions, the resulting certificates are remarkably tight. For example, for a large Vision Transformer with validation error $0.2$, our bound yields a certificate of $0.3$.
\end{itemize}

To the best of our knowledge, this is the first work to provide non-vacuous guarantees at this large scale without altering the trained models. The tightness of our certificates suggests that meaningful, data-dependent generalization guarantees for modern NNs are practically realizable without requiring large held-out test sets.

\textit{Organization:} We present our novel bounds in Section~\ref{sec-Error-bounds}, accompanied with more detailed comparisons. Section~\ref{sec-evaluation} contains our empirical evaluation for some pretrained NNs. Section~\ref{sec-Conclusion} concludes the paper. A comprehensive survey about related work is presented in Appendix~\ref{sec-related-work}.

\begin{table*}[tp]
\caption{Recent approaches for analyzing generalization error. $\checkmark$ means ``Required'' or ``Yes''. The upper part shows the required assumptions about differrent aspects, e.g., hypothesis space, loss function, training or finetuning. The lower part reports non-vacuousness in different situations. }
\label{tab-Recent-approaches}
\addtolength{\tabcolsep}{-0.5em}
\begin{tabular}{lccccccc}
\hline 
\textbf{Approach} & \footnotesize{\textbf{Weight norm}} & \footnotesize{\textbf{Alg. Stability}} & \footnotesize{\textbf{Alg. Robustness}} & \footnotesize{\textbf{Mutual Info}}&  \multicolumn{2}{c}{\footnotesize{\textbf{PAC-Bayes}}} & \footnotesize{\textbf{Ours}} \\
 & \tiny{\cite{golowich2020RC}} &  \tiny{\cite{li24algorithmic}} & \tiny{\cite{kawaguchi22RobustGen}}   & \tiny{\cite{nadjahi24slicingMI}}  & \tiny{\cite{mustafa2024StochasticNN}}  & \tiny{\cite{lotfi24unlockingLLM}} & \\
\hline
\small{\textbf{Requirement:}} & & & & & & & \\
\hspace{5pt}\small{Model compressibility}  &  & & & $\checkmark$ &  $\checkmark$ & $\checkmark$ &  \\
\hspace{5pt}\small{Train or finetune} & & & & $\checkmark$ &  $\checkmark$ &  $\checkmark$ &  \\
\hspace{5pt}\small{Lipschitz loss}  & $\checkmark$ & $\checkmark$ &  & $\checkmark$ & & & \\
\hspace{5pt}\small{\textit{Finite} hypothesis space}  & & & & & &  $\checkmark$ &  \\
\textcolor{blue}{\small{\textbf{Non-vacuousness} for:}} & &  &  & & & &  \\
\hspace{5pt}\small{\textit{Stochastic} models} only &  &  $\checkmark$ &  & $\checkmark$ &  $\checkmark$ & & \\ 
\hspace{5pt}\small{Trained models} & & & & & & $\checkmark$ & $\checkmark$\\
\hspace{5pt}\small{Training size $>$ 1 M}  & & & & &  & $\checkmark$ & $\checkmark$ \\
\hspace{5pt}\small{Model size $>$ 600 M}  & & & & &  & $\checkmark$ & $\checkmark$\\
 \hline
\end{tabular}
\vspace{-10pt}
\end{table*}

\section{Error bounds} \label{sec-Error-bounds}

In this section, we present  novel bounds for the error of a given model. The first bound provides a general form which depends on the complexity of the data distribution and the trained model. This bound cannot be exactly computed, but serves as the theoretical foundation. The last bound provides an explicit error estimate, which can be computed directly from any given dataset.

\textit{Notations:} $\mS$ often denotes a dataset and $| \mS |$ denotes its size/cardinality. $\Gamma$ denotes a partition of the data space. $[K]$ denotes the set $\{1, ..., K\}$ of natural numbers at most $K$. 

Consider a hypothesis  (or model) $\vh: \gX \rightarrow \gY$ which maps from an input space $\gX$ to an output space $\gY$, and a  loss function $\ell: (\vh, (\vx,\vy)) \mapsto l \in \R$, where $(\vx,\vy) \in \gX \times \gY$. Each $\ell(\vh,\vz)$ tells the loss (or quality) of $\vh$ at an instance $\vz \in \gZ: = \gX \times \gY$. Given a distribution $P$ defined on $\gZ$, the quality of $\vh$ is measured by its \textit{expected loss} $F(P, \vh) = \E_{\vz \sim P}[\ell(\vh,\vz)]$.  Quantity $F(P, \vh)$ tells the generalization ability of model $\vh$; a smaller $F(P, \vh)$ implies  better generalization on unseen data. 

For analyzing generalization ability, we are often interested in estimating (or bounding) $F(P, \vh)$. Sometimes this expected loss is compared with the \emph{empirical loss}  of $\vh$ on a data set $\mS = \{\vz_1, ..., \vz_n\} \subseteq \gZ$, which is defined as $F(\mS, \vh) =  \frac{1}{n} \sum_{\vz \in \mS} \ell(\vh,\vz)$. Note that a small $F(\mS, \vh)$ does not neccessarily imply good generalization of $\vh$, since overfitting may appear. Therefore, our ultimate goal is to estimate $F(P, \vh)$ directly.

Let $\Gamma(\gZ) := \bigcup_{i=1}^{K } \gZ_i$ be a partition of $\gZ$ into $K $ disjoint nonempty subsets. Let $\mS_i = \mS \cap \gZ_i $, $n_i = | \mS_i |$ be the number of samples  falling into $\gZ_i$, and $n = \sum_{j=1}^K n_j$. Let $\mT = \{ i \in [K ] : n_i >0 \}$ contain the indices of areas in which some samples of $\mS$ appear, $a_i(\vh) = \E_{\vz}[\ell(\vh,\vz) | \vz \in \gZ_i]$ as the expected (local) loss of $\vh$ in area $\gZ_i$ for each $i \in [K]$, and $a_o = \max_{j \notin \mT} a_j(\vh)$.

\subsection{General bound}

The first result incorporates some properties of the data distribution and the trained model. 

\begin{theorem}\label{thm-gen-train-small-K}
Given  a partition $\Gamma$ and a bounded nonnegative loss $\ell$, consider a model $\vh$ which may depend on a dataset $\mS$ with $n$ i.i.d. samples from distribution $P$. Denote $p_i = \Pr_{\vz \sim P}(\vz \in \gZ_i)$ as the probability measure of area $\gZ_i$ for $i \in [K]$,  and $u = \sum_{i=1}^K \gamma n p_i(1+ \gamma n p_i)$. For  any constants $\gamma \ge 1$, $\delta_1 \ge \exp(- \frac{u \ln\gamma}{4n-3})$ and $\delta_2 >0$, we have the following with probability at least $1 -\delta_1 -\delta_2$:
\begin{equation}
\label{thm-gen-train-small-K-eq}
F(P, \vh) \le F(\mS,\vh)  + C\sqrt{\frac{u}{2n^2} \ln\frac{1}{\delta_1} } + g(\Gamma,\vh,\delta_2)
\end{equation}
where $g(\Gamma,\vh,\delta_2) = \frac{\sqrt{\ln(2K/\delta_2)}}{n} \sum\limits_{i \in \mT} \sqrt{n_i} \left(a_o + \sqrt{2} a_i(\vh)\right)   + \frac{2\ln(2K/\delta_2)}{n} (a_o | \mT| +  \sum\limits_{i \in \mT} a_i(\vh) )$ and $C = \sup_{\vz \in \gZ} \ell(\vh,\vz)$.
\end{theorem}

Theorem~\ref{thm-gen-train-small-K} establishes that the expected loss $F(P,\vh)$ of a trained model cannot deviate significantly from its empirical loss $F(\mS,\vh)$, with high probability. The deviation is controlled by two additive terms: a distribution-dependent uncertainty term and a partition-dependent alignment term. Together, these terms quantify how sampling noise, data geometry, and local model behavior jointly shape generalization. Crucially, the bound makes explicit that generalization is not governed solely by sample size or global capacity measures, but by how the model’s local errors interact with the geometric structure of the data distribution.

We highlight several key implications of this result.

\begin{itemize}
\item \textit{Geometry-aware distributional complexity:} The quantity $u= \gamma n + (\gamma n)^2 \sum_{i=1}^K p_i^2$ captures a notion of \emph{distributional complexity} induced by the partition $\Gamma$. Unlike classical complexity measures that depend on hypothesis classes, $u$ reflects how probability mass is distributed across regions of the input space. From a geometric perspective, distributions (e.g., low-variance Gaussians or data supported near low-dimensional manifolds)  that concentrate mass in a small number of regions lead to larger values of $\sum_i p_i^2$ and hence larger $u$. In contrast, geometrically diffuse distributions distribute mass more evenly across partitions, resulting in smaller $u$ and tighter bounds. The term $|\mT|$ further provides a coarse geometric summary of how many regions are meaningfully populated by the data. Together, these terms make the bound explicitly sensitive to the \emph{shape}, \emph{concentration}, and \emph{support geometry} of $P$, a dependence that is largely absent from prior generalization theories.

\item \textit{Alignment between data geometry and local loss:} A crucial feature of the bound is the appearance of the alignment term, $\mathrm{Align} := \sum_{i\in \mT} a_i(\vh)\sqrt{n_i/n}$, which can be interpreted as a dot product between local errors and the square root of local data mass. Geometrically, this term penalizes regions where the model incurs large loss precisely in areas of high probability density. A well-chosen partition induces a near-orthogonality between these two vectors: high-density regions exhibit small local errors, while regions with larger errors are confined to areas of low mass. Conversely, poor alignment corresponds to a geometric mismatch between the model’s loss landscape and the data distribution, resulting in large contributions to the bound, as illustrated in Figure~\ref{fig-geometry-align-example}. This mechanism highlights that generalization hinges on \emph{where} errors occur in the input space, not merely on their global average—a geometric insight not captured by classical bounds. 

\item \textit{Partition as a geometric lens:} The partition $\Gamma$ plays the role of a geometric lens through which both the data distribution and the model’s loss landscape are discretized and compared. Finer partitions can resolve more detailed local structure, but may increase $|\mT|$ and exacerbate  $g(\Gamma,\vh,\delta_2)$. Coarser partitions reduce variance but may obscure important geometric heterogeneity. The theorem thus formalizes a trade-off: \textit{effective generalization guarantees arise when the partition aligns with intrinsic data geometry, capturing local regularity of both $P$ and $\vh$.} In this sense, the bound implicitly favors partitions that respect the underlying manifold or clustering structure of the data.

\item \textit{Model-dependent guarantee with mild assumptions:} Unlike stability  \citep{li24algorithmic,dong2025exactly}, robustness \citep{xu2012robustnessGeneralize,kawaguchi22RobustGen}, or Radermacher complexity  \citep{bartlett2017SpectralMarginDNN,galanti2023normDNN} approaches, the bound makes no assumption on the hypothesis class or the learning algorithm. It only requires i.i.d. and bounded loss, which are commonly used in theoretical work.
\end{itemize}

Overall, Theorem~\ref{thm-gen-train-small-K} reframes generalization as a geometric phenomenon: \textbf{test error is controlled by how well the model’s local loss structure aligns with the data distribution across the input space.} By making this interaction explicit, the bound provides both a guarantee and a diagnostic tool for understanding when and why a trained model generalizes.

\begin{figure*}[h!]
    \centering
    \includegraphics[width=0.85\linewidth]{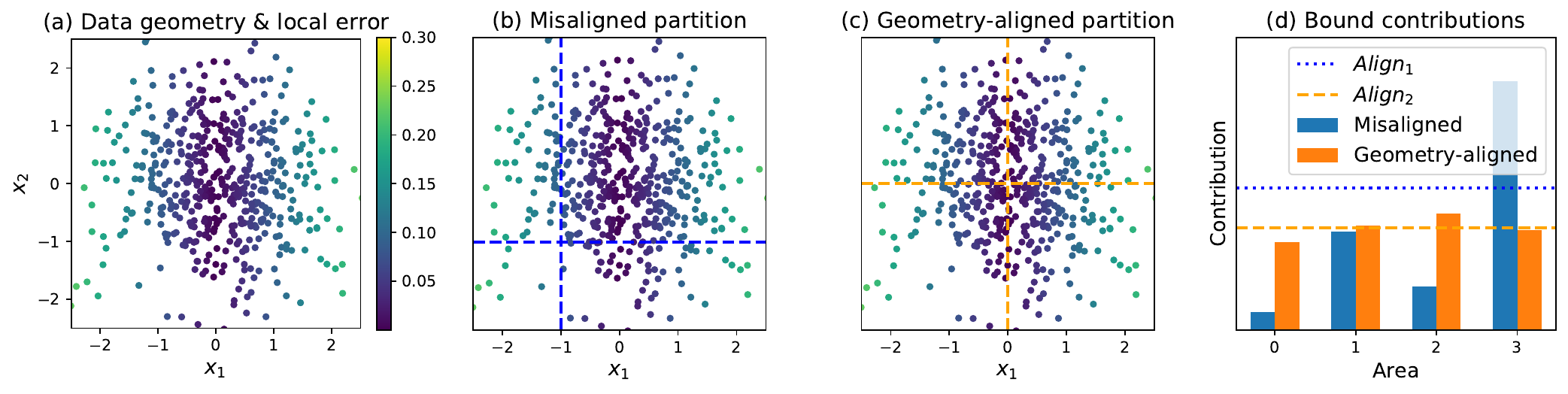}
    \caption{Data geometry, partition, and alignment. (a) 500 samples generated from a synthetic distribution and their prediction errors shown by colors. (b) A bad partition that devides the data space into uneven regions. (c) A geometry-aligned partition that concentrates error. (d) Contributions from local regions by different partitions, and alignments between data geometry and local errors. $Align_1$ is the result of using a misaligned parition, while $Align_2$ comes from a well-aligned one.}
    \label{fig-geometry-align-example}
\end{figure*}

\begin{remark}
It is worth noticing the similarity between our bound (\ref{thm-gen-train-small-K-eq}) and robustness-based bounds in \citep{kawaguchi22RobustGen,than2025gentle}. $F(\mS,\vh) + g(\Gamma,\vh,\delta_2)$ is the common part in those bounds. Our bound~(\ref{thm-gen-train-small-K-eq}) contains $u$ that encodes the complexity of the data distribution, whereas the bounds in \citep{kawaguchi22RobustGen,than2025gentle} use a robustness quantity that measures the sensitivity of the loss w.r.t. a change in the input. While prior bounds are not amenable to be exactly computed from a training set, our bound enables to easily derive a computable, non-vacuous bound (below). This is the main advantage of bound (\ref{thm-gen-train-small-K-eq}).
\end{remark}

One limitation of bound (\ref{thm-gen-train-small-K-eq}) is that it is not diminishing as $n$ increases while fixing the partition size $K$. It can be seen from the second term, i.e., $ C\sqrt{0.5(\frac{\gamma}{n} + \gamma^2\sum_{i=1}^K  p_i^2)  \ln\frac{1}{\delta_1} }$. Luckily, this issue can be easily fixed by allowing $K$ to increase with $n$, owing to the following result.

\begin{corollary}\label{cor-gen-train-small-K-continuous}
Given the notations in Theorem \ref{thm-gen-train-small-K}, consider a continuous distribution $P$ supported on a convex domain $\gZ$. 
For any $K>0, \delta_2 >0$, $\delta_1 \ge \exp\left(-{(\gamma n + \gamma^2 n^2/K) (\ln\gamma)}/{(4n-3)}\right)$,  with probability at least $1 -\delta_1 -\delta_2$, we have: 
$F(P, \vh) \le F(\mS,\vh)  + C\sqrt{0.5\left(\frac{\gamma}{n} + \frac{\gamma^2}{K}\right) \ln\frac{1}{\delta_1} } +   g(\Gamma,\vh,\delta_2) $. 
\end{corollary}

\begin{remark}[Convergence rate]
This result suggests that by choosing $K = O(n^\beta)$ for $\beta \in [0,1]$, the test error is bounded by $F(\mS,\vh) + O(n^{-\beta/2})$. Note that there is a tradeoff between $K$ and $g$. A large $K$ can potentially produce a large $g$, as evidenced in our ablation later on real data. Therefore a balanced choice for $K$ seems to be $O(n^{1/2})$, making our bound  scale as $O(n^{-1/4})$. Note that this convergence rate seems to be sub-optimal, and hence leaves open room for future improvement.
\end{remark}

\begin{proof}[Proof sketch for Theorem \ref{thm-gen-train-small-K} and \textbf{technical novelty}]
The detailed proof appears in Appendix~\ref{app-Proofs-main-results}. We focus on bounding the probability $\Pr\left(F(P,\vh) - F(\mS,\vh) \ge \phi\right)$, for some gap $\phi$. Note that $F(P,\vh)- F(\mS,\vh) = A +B$, where $A = F(P,\vh) - \sum_i \frac{n_i}{n} a_i(\vh)$ and $B=\sum_i \frac{n_i}{n} a_i(\vh) - F(\mS,\vh)$. Therefore, our proof estimates $\Pr(A \ge g) $ and $\Pr(B \ge t)$ for some constant $t$. Once they are known, we can use the union bound to obtain a bound on $\Pr\left(F(P,\vh) - F(\mS,\vh) \ge g+ t \right)$ as desired. We use a result from \citep{kawaguchi22RobustGen} to bound $\Pr(A \ge g)$. The remaining task is to estimate $\Pr(B \ge t)$, which is \textbf{the main challenge}. This challenge requires approximating an intractable quantity from a data set.

We resolve this challenge by developing Theorem~\ref{thm-gen-general-data-part}. Its proof contains three main steps: 

1. First we show $\Pr(B(\vh) \ge t ) \le  e^{-y t}  \E_{\vv,\vn}\left[ \E_{\mS}\left[e^{y B(\vv)} | \vv, \vn\right] \right] $, for  $\vn = \{n_1, ..., n_K\}$ and  some  $y$. This is the \textit{key novelty}, where we transform error estimation for $\vh$ into estimation of $\E_{\mS}\left[e^{y B(\vv)} | \vv, \vn\right]$ for a fixed model $\vv$ and then integrating $\vv$ out via $\E_{\vv,\vn}$.

2. We next estimate $\E_{\mS}\left[e^{y B(\vv)} | \vv, \vn\right]$.  Overall, we make sure that $\E_{\mS}\left[e^{y B(\vv)} | \vv, \vn\right] \le e^{\psi(y,\vn)}$, for some function $\psi(y,\vn)$ which does not depend on $\vv$. As a result $\Pr(B(\vh) \ge t ) \le e^{-y t} \mathbb{E}_{\vn}e^{\psi(y,\vn)}$.

3. The last step is to bound $\mathbb{E}_{\vn}e^{\psi(y,\vn)}$. This requires us to develop \textit{novel} analyses for small random variables in Appendix~\ref{app-Supporting-theorems}. A suitable choice for $t, y$ completes our proof. 
\end{proof}

\subsection{Computable bound}

It is worth noticing that bound (\ref{thm-gen-train-small-K-eq}) contains some unknown quantities, e.g., $u$ and $a_i$'s, which cannot be computed exactly. This is the main limitation.  The following bound overcomes such a limitation.

\begin{theorem}\label{thm-gen-train-small-K-any-distribution}
Given the notations and assumption in Theorem~\ref{thm-gen-train-small-K}, for  any constants $\gamma \ge 1, \delta >0$ and  $\alpha \in [0, \frac{\gamma n(K+\gamma n)}{K(4n-3)}]$,  we have the following with probability at least $1- \gamma^{-\alpha}  -\delta$:
\begin{equation}
\label{thm-gen-train-small-K-any-distribution-eq}
F(P, \vh) \le F(\mS,\vh)  + C\sqrt{\hat{u}\alpha \ln\gamma } + g_2(\delta/2)
\end{equation}
where {\small $\hat{u} = \frac{\gamma}{2n} +  \frac{\gamma^2 }{2}\sum_{i=1}^K \left( \frac{n_i}{n} \right)^2 + \gamma^2 \sqrt{\frac{2}{n}\ln\frac{2K}{\delta}}$, $g_2(\delta) =  \frac{C(1 + \sqrt{2})\sqrt{\ln(2K/\delta)}}{n} \sum_{i \in \mT} \sqrt{n_i}  + \frac{4C | \mT| \ln(2K/\delta)}{n}$}.
\end{theorem}

This result follows directly from Theorem \ref{thm-gen-train-small-K}, where $g_2(\delta)$ serves as a simplified (and generally looser) surrogate for the earlier term $g(\Gamma,\vh,\delta_2)$. Bound~\ref{thm-gen-train-small-K-eq} is not exactly computable, due to $u$ and $g$. In contrast, $\hat{u}$ and $g_2$ are computable approximations of $u$ and $g$, respectively. It is important to note that while $\hat{u}$ preserves the essential role of $u$ in capturing distributional complexity, $g_2$ no longer retains the fine-grained local structure encoded in $g$. Therefore, Bound (\ref{thm-gen-train-small-K-any-distribution-eq}) should serve as providing certificate for a model, rather than understanding generalization.

A key advantage of this result is that \textit{the bound can be evaluated using only the training set}. Of couse, some held-out samples can contribute more. Once a choice of $K$ and a partition $\Gamma$ is fixed, we can compute the counts $n_i$, identify the set $\mT$, and directly evaluate Bound~(\ref{thm-gen-train-small-K-any-distribution-eq}). This ease of computation makes the bound particularly practical and appealing for large-scale applications.

\textbf{A theoretical comparison with closely related bounds:}  Although many model-dependent bounds  \citep{kawaguchi22RobustGen,than2025gentle,biggs2022nonvacuousBound,viallard2024PacBayes,lotfi24NonVacuousLLM,lotfi24unlockingLLM} have been proposed, our bound (\ref{thm-gen-train-small-K-any-distribution-eq}) has various advantages: 
\begin{itemize} 
\item \textit{Mild assumption:} Our bound does not require stringent assumptions as in prior ones. Some prior bounds  require stability \citep{li24algorithmic,lei2020stabilitySGD} or robustness \citep{xu2012robustnessGeneralize,kawaguchi22RobustGen,sokolic2017robustDNN} of the learning algorithm. Those assumptions are often violated in practice, e.g. for the appearance of adversarial attacks \citep{zhou2022adversarial}. Some theories \citep{lotfi24NonVacuousLLM,lotfi24unlockingLLM} assume that the hypothesis class is finite, which is restrictive. In contrast, our bound requires only   i.i.d. and loss' boundedness assumptions, which are commonly used in prior bounds. 

\item \textit{Easy and cheap evaluation:} An evaluation of our bound (\ref{thm-gen-train-small-K-any-distribution-eq}) will be simple and does not require any modification to the model $\vh$ of interest, while requiring a low complexity (see Appendix~\ref{app-Computational-complexity}). Those are crucial advantages. Many prior theories require intermediate steps to change the model of interest into a suitable form. For example, state-of-the-art  methods to compress NNs are required for \citep{zhou2019CompressionBound,lotfi2022pac,nadjahi24slicingMI}; quantization for a model is required for \citep{lotfi24NonVacuousLLM,lotfi24unlockingLLM}; finetuning (e.g. SubLoRA) is required for \citep{lotfi24NonVacuousLLM,lotfi24unlockingLLM}. Those facts suggests that evaluations for prior bounds are often expensive. Besides, many prior model-dependent bounds \citep{xu2012robustnessGeneralize,kawaguchi22RobustGen,than2025gentle}  cannot be exactly computed.

\item \textit{No change to the model:} Most prior non-vacuous bounds \citep{zhou2019CompressionBound,dziugaite2017computingBounds,lotfi24NonVacuousLLM,lotfi24unlockingLLM} require extensively compressing (or quantizing) model $\vh$ of interest and then retraining/finetuning the compressed version. Sometimes the compression step is too restrictive and produces low-quality models \citep{lotfi24NonVacuousLLM}. Therefore, a modification will change model $\vh$ and hence \textbf{\textit{those bounds do not directly provide guarantees for the generalization ability of $\vh$}}. In contrast, our bound (\ref{thm-gen-train-small-K-any-distribution-eq}) does not require any change to model $\vh$, and hence directly provides a guarantee for $\vh$.
\end{itemize}

\begin{remark}
There is a nonlinear relationship between $K$ and the uncertainty term $\text{Unc}(\Gamma) = C\sqrt{\hat{u}\alpha \ln\gamma } + g_2(\delta/2)$ in our bound. A partition with a larger $K$ can make the sum $\sum_{i=1}^K \left( {n_i}/{n} \right)^2$ smaller, as the samples can be spread into more  areas. However a larger $K$ can make $g_2(\delta)$ larger. Therefore, we should not choose too large $K$. On the other hand, a small $K$ can make the sum $\sum_{i=1}^K \left( {n_i}/{n} \right)^2$ large, since more samples can appear in each area $\gZ_i$ and enlarge ${n_i}/{n}$. Therefore, we should not choose too small $K$.
\end{remark}

\begin{remark}[Partition choice]
The partition $\Gamma$ is a key design choice in our framework, as it strongly influences both the tightness and informativeness of the bounds, but it must be specified independently of the trained model $\vh$ and the training set $\mS$ to ensure validity. Within this constraint, $\Gamma$ may depend on the data distribution (P), the hypothesis class, or the learning algorithm, allowing the use of domain knowledge to capture meaningful structure. In practice, one can construct $\Gamma$ using auxiliary data sources independent of $\mS$ (e.g., unlabeled data or a held-out split). Moreover, representation-aware partitions are both valid and often advantageous: defining $\Gamma$ in the feature space of a good pretrained model often yields semantically meaningful regions, making locality easier to capture than in raw input spaces such as pixels.
\end{remark}

\section{Empirical evaluation} \label{sec-evaluation}

In this section, we present an extensive empirical evaluation of our bounds. We first investigate the strength of the guarantees for test error of trained large-scale models without any modification. We then examine the key predictive factors that control the bounds and model's generalization. Appendix~\ref{app-simple-models} provides more investigations about traditional ML models and Appendix~\ref{app-Comparation-prior-bounds} compares with closely related bounds which cannot be exactly computable.

\begin{table*}[h!]
\centering
\caption{Upper bounds on the true error (in \%) of 30 deep NNs which were pretrained on ImageNet dataset.  The second column presents the model size, the third column contains the test accuracy at Top 1, as reported by Pytorch. The last three columns report our estimates about the true error, with a certainty at least 95\%.}
\label{tab:sup-imagenet-bound-train-only}
\addtolength{\tabcolsep}{-0.2em}
{\small 
\begin{tabular}{|lcccc|ccc|}
\hline 
\multirow{2}{*}{\textbf{Model}} & \multirow{2}{*}{\textbf{\#Params} (M)} & \multirow{2}{*}{\textbf{Training error}}  & \multirow{2}{*}{\textbf{Acc@1}} &  \multirow{2}{*}{\textbf{Test error}} &  \multicolumn{3}{c|}{\textbf{Error bound (\ref{thm-gen-train-small-K-any-distribution-eq})} } \\
& & & & & $\Gamma_0$ & $\Gamma_1$ & $\Gamma_2$ \\
\hline
ResNet50 V1           & 25.6  & 13.121 & 76.130 & 23.870 & 49.772  & 45.088    & 39.613    \\
ResNet101 V1          & 44.5  & 10.502 & 77.374 & 22.626 & 47.153  & 42.469    & 36.994    \\
ResNet152 V1          & 60.2  & 10.133 & 78.312 & 21.688 & 46.784  & 42.100    & 36.625    \\
ResNet50 V2           & 25.6  & 8.936  & 80.858 & 19.142 & 45.587  & 40.903    & 35.428    \\
ResNet101 V2          & 44.5  & 6.008  & 81.886 & 18.114 & 42.659  & 37.975    & 32.500    \\
ResNet152 V2          & 60.2  & 5.178  & 82.284 & 17.716 & 41.829  & 37.145    & 31.670    \\ \hdashline
SwinTransformer B     & 87.8  & 6.464  & 83.582 & 16.418 & 43.115  & 38.431    & 32.956    \\
SwinTransformer B V2  & 87.9  & 6.392  & 84.112 & 15.888 & 43.043  & 38.359    & 32.884    \\
SwinTransformer T     & 28.3  & 9.992  & 81.474 & 18.526 & 46.643  & 41.959    & 36.484    \\
SwinTransformer T V2  & 28.4  & 8.724  & 82.072 & 17.928 & 45.375  & 40.691    & 35.216    \\ \hdashline
VGG13                 & 133.0 & 18.456 & 69.928 & 30.072 & 55.107  & 50.423    & 44.948    \\ 
VGG13 BN              & 133.1 & 19.223 & 71.586 & 28.414 & 55.874  & 51.190    & 45.715    \\
VGG19                 & 143.7 & 16.121 & 72.376 & 27.624 & 52.772  & 48.088    & 42.613    \\
VGG19 BN              & 143.7 & 15.941 & 74.218 & 25.782 & 52.592  & 47.908    & 42.433    \\ \hdashline
DenseNet121           & 8.0   & 15.631 & 74.434 & 25.566 & 52.282  & 47.598    & 42.123    \\
DenseNet161           & 28.7  & 10.48  & 77.138 & 22.862 & 47.131  & 42.447    & 36.972    \\
DenseNet169           & 14.1  & 12.395 & 75.600 & 24.400 & 49.046  & 44.362    & 38.887    \\
DenseNet201           & 20.0  & 9.806  & 76.896 & 23.104 & 46.457  & 41.773    & 36.298    \\ \hdashline
ConvNext Base         & 88.6  & 5.209  & 84.062 & 15.938 & 41.860  & 37.176    & 31.701    \\
ConvNext Large        & 197.8 & 3.846  & 84.414 & 15.586 & 40.497  & 35.813    & 30.338    \\ \hdashline
RegNet Y 128GF linear & 644.8 & 9.032  & 86.068 & 13.932 & 45.683  & 37.532    & 32.057    \\
RegNet Y 32GF linear  & 145.0 & 10.558 & 84.622 & 15.378 & 47.209  & 40.999    & 35.524    \\
RegNet Y 32GF V2      & 145.0 & 3.761  & 81.982 & 18.018 & 40.412  & 39.094    & 33.619    \\ 
RegNet Y 32GF e2e     & 145.0 & 7.127  & 86.838 & 13.162 & 43.778  & 42.525    & 37.050    \\
RegNet Y 128GF e2e    & 644.8 & 5.565  & 88.228 & 11.772 & 42.216  & 35.728    & 30.253    \\ \hdashline
VIT H 14 linear       & 632.0 & 9.951  & 85.708 & 14.292 & 46.602  & 46.936    & 41.461    \\
VIT B 16 linear       & 86.6  & 14.969 & 81.886 & 18.114 & 51.620  & 37.883    & 32.408    \\
VIT L 16 linear       & 304.3 & 11.003 & 85.146 & 14.854 & 47.654  & 41.918    & 36.443    \\
VIT B 16 V1           & 86.6  & 5.916  & 81.072 & 18.928 & 42.567  & 42.970    & 37.495    \\
VIT L 16 V1           & 304.3 & 3.465  & 79.662 & 20.338 & 40.116  & 35.432    & 29.957    \\ \hline
\end{tabular}}
\end{table*}

\subsection{Guarantees for large-scale pretrained models}

\textbf{Models:} We use 30 modern NN models\footnote{https://pytorch.org/vision/stable/models.html} which were pretrained by Pytorch on the ImageNet dataset with 1,281,167 images. All models are multiclass classifiers. We use the ImageNet training set exclusively to compute Bound (\ref{thm-gen-train-small-K-any-distribution-eq}).

\textbf{Baselines:} While many model-dependent bounds exist, we exclude them from direct comparison for the following reasons: (1) several bounds~\citep{kawaguchi22RobustGen,than2025gentle,Luxburg04Lipschitz,hou2023instanceGen} cannot be computed exactly from the training set alone; (2) all norm-based bounds~\citep{bartlett2017SpectralMarginDNN,arora2018strongerBounds,golowich2020RC,graf2022measuring,galanti2023normDNN} are vacuous even for relatively small networks; and (3) certain PAC-Bayes bounds~\citep{biggs2022nonvacuousBound} apply only to shallow or specialized architectures, while others~\citep{zhou2019CompressionBound,dziugaite2017computingBounds,lotfi2022pac} estimate $\E_{\hat{\vh}}[F(P, \hat{\vh}) ]$, the expected test error of a stochastic model. Those bounds and the ones in \citep{lotfi24NonVacuousLLM,lotfi24unlockingLLM} require substantial modifications to the original network. Such requirements render them incompatible with our evaluation setting.

\textbf{Experimental settings:} We fix $K=200, \delta = 0.01, \alpha =100, \gamma = 0.04^{-1/\alpha}$. This choice means that our bound is correct with probability at least 95\%. The upper bound (\ref{thm-gen-train-small-K-any-distribution-eq}) for each model was computed with 5 random seeds. We use the 0-1 loss function,  meaning that our bound directly estimates the true classification error. 

\textbf{Partitions:} We investigated three different ways to define the partition. \textbf{Baseline} $\Gamma_0$ whose centroids are initialized randomly in the input space; \textbf{Representation-based} $\Gamma_1$ whose centroids are initialized randomly in the feature space of a pretrained ResNet-18; 
\textbf{Validation-based} $\Gamma_2$ whose centroids are obtained via K-means on the ImageNet validation set. A sample will be assigned to the area whose centroid is closest. Note that $\Gamma_0$ is the most naive way when we do not have any knowledge about the data space, while $\Gamma_2$ exploits prior knowledge through a held-out dataset and hence should align better with the data distribution. Those partitions reflect different scenarios in practice.

\textbf{Results:} The overall results are reported in Table~\ref{tab:sup-imagenet-bound-train-only}. One can observe that our bound for all models are all non-vacuous even for the non-optimized choices of some parameters. We observe that $\Gamma_1$ consistently improves over $\Gamma_0$, indicating that representation-aligned partitions enhance alignment with data geometry and yield tighter bounds. $\Gamma_2$ achieves the best results, better than those of $\Gamma_0$ by ~10\% in absolute value across all models, highlighting the benefit of using independent data to construct geometry-aligned partitions. 

When the parameters are well chosen, one can obtain much better error bounds. Note that non-vacuousness of our bound holds true for a large class of deep NN families, some of which have more than 630M parameters. To the best of our knowledge,  bound (\ref{thm-gen-train-small-K-any-distribution-eq})  is the first theoretical bound which is non-vacuous at such a large scale, without  any modification to the trained models.

\subsection{Predictive factors of generalization}\label{app-understanding-generalization}

We next investigate \textit{how predictive are the bounds and where are predictive factors of generalization? }These are important when one wants to understand the main factors that lead to better generalization of a model. To this end, we focus on the following quantities in Bound (\ref{thm-gen-train-small-K-eq}): \\
\begin{eqnarray}
{Align} &=& \sum_{i \in T} a_i(h)\sqrt{n_i/n}, \qquad \qquad 
{Fair} = \sum_{i \in T} a_i(h) \\
{Behavior} &=& {Align} \sqrt{\frac{2 \ln(2K/\delta_2)}{n}}    + {Fair} \frac{2\ln(2K/\delta_2)}{n}
\end{eqnarray}


Note that $Align$ tells how well the model's local error can match with the data distribution. A better model should align better with the distribution's complexity, hence making $Align$ smaller. Meanwhile $Fair$ tells the macro-level error of model $h$. It also suggests how fair for different local areas the model is. Finally, $Behavior$ unifies them to be an important part of Bound (\ref{thm-gen-train-small-K-eq}).

We use $K=200, \delta_2=0.01$ and the ImageNet validation set to compute those quantities. The results for 6 pretrained models are reported in Table~\ref{tab:understanding-generalization-Align-Fair}. We can observe that all of those quantities have extremely high correlations to the test error. $Align$ has the highest correlation, but $Fair$ has the lowest one. Appendix~\ref{app-understanding-generalization} provides more insights into the high correlation of $Align$. These results demonstrate that $Align$ can exhibit the quality of a model, and can be an accurate indicator for comparison between two models. The strong correlation between \textit{the data-model behavior alignment} (as measured by $Align$) and the test error suggests that such an alignment may be a critical factor in generalization for modern large models.

\begin{table}[tp]
\centering
\caption{Correlation between different factors with test error.}
\addtolength{\tabcolsep}{-0.25em}
{
\begin{tabular}{|l|c|c|c|c|}
\hline
Model                     & Test error & Align   & Fair    & Behavior \\ \hline
ResNet18 V1               & 0.302    & 3.289 & 64.683 & 0.014  \\
ResNet101 V1              & 0.226    & 2.433 & 46.086 & 0.011  \\
ResNet152 V2              & 0.177    & 1.931 & 38.820 & 0.009  \\
DenseNet201               & 0.231    & 2.383 & 45.875 & 0.010  \\
SwinTransformer B         & 0.159    & 1.723 & 33.167 & 0.008  \\
VIT B 16 linear           & 0.181    & 2.141 & 43.766 & 0.009  \\ \hline
\textbf{Correlation to test error} &            & \textbf{0.986} & \textbf{0.964} & \textbf{0.984} \\ \hline
\end{tabular}
}
\label{tab:understanding-generalization-Align-Fair}
\end{table}

\subsection{Ablation study}

\textbf{Parameters}: Note that our bound depends on the choice of some parameters. Figure~\ref{fig:uncertainty-sum-K-change} reports the changes of $\sum_{i=1}^K \left( \frac{n_i}{n} \right)^2$ as the partition $\Gamma$ changes. We can see that this quantity tends to decrease as we divide the input space into more small areas. Meanwhile, Figure~\ref{fig:uncertainty-Alpha-change} reports the uncertainty term, as either $\alpha$ or $K$ changes. Observe that a larger $K$ can increase the uncertainty fast, while an increase in $\alpha$ can gradually decrease the uncertainty. Those figures enable an easy choice for the parameters in our bound.

\begin{figure}[tp]
    \centering
   \includegraphics[width=.7\linewidth]{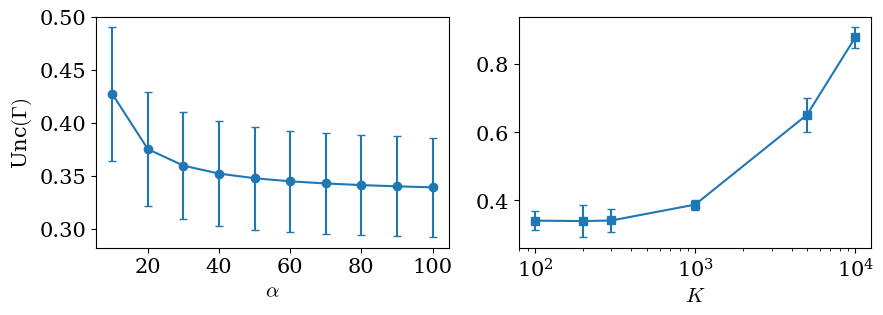}
       \caption{The uncertainty $\text{Unc}(\Gamma) = C\sqrt{\hat{u}\alpha \ln\gamma } + g_2(\delta/2)$ as (left) $\alpha$ changes, for fixed $K = 200$, and (right) $K$ changes.} 
       \label{fig:uncertainty-Alpha-change}
    
     \begin{minipage}{.5\textwidth}
        \includegraphics[width=\linewidth]{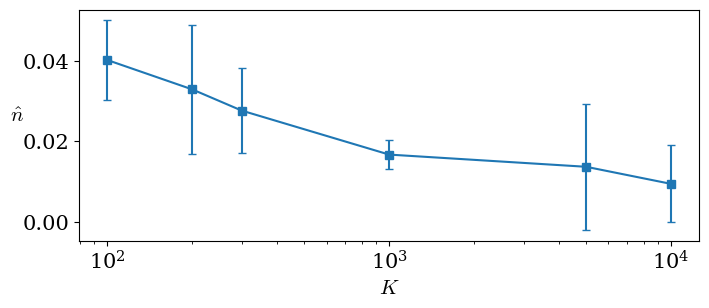}
        \caption{Distributional complexity $\hat{n} = \sum_{i} \left( \frac{n_i}{n} \right)^2$ observed from ImageNet, as $K$ changes.}
        \label{fig:uncertainty-sum-K-change}
        \end{minipage} \hspace{5pt}
    \begin{minipage}{.45\textwidth}
        \includegraphics[width=\linewidth]{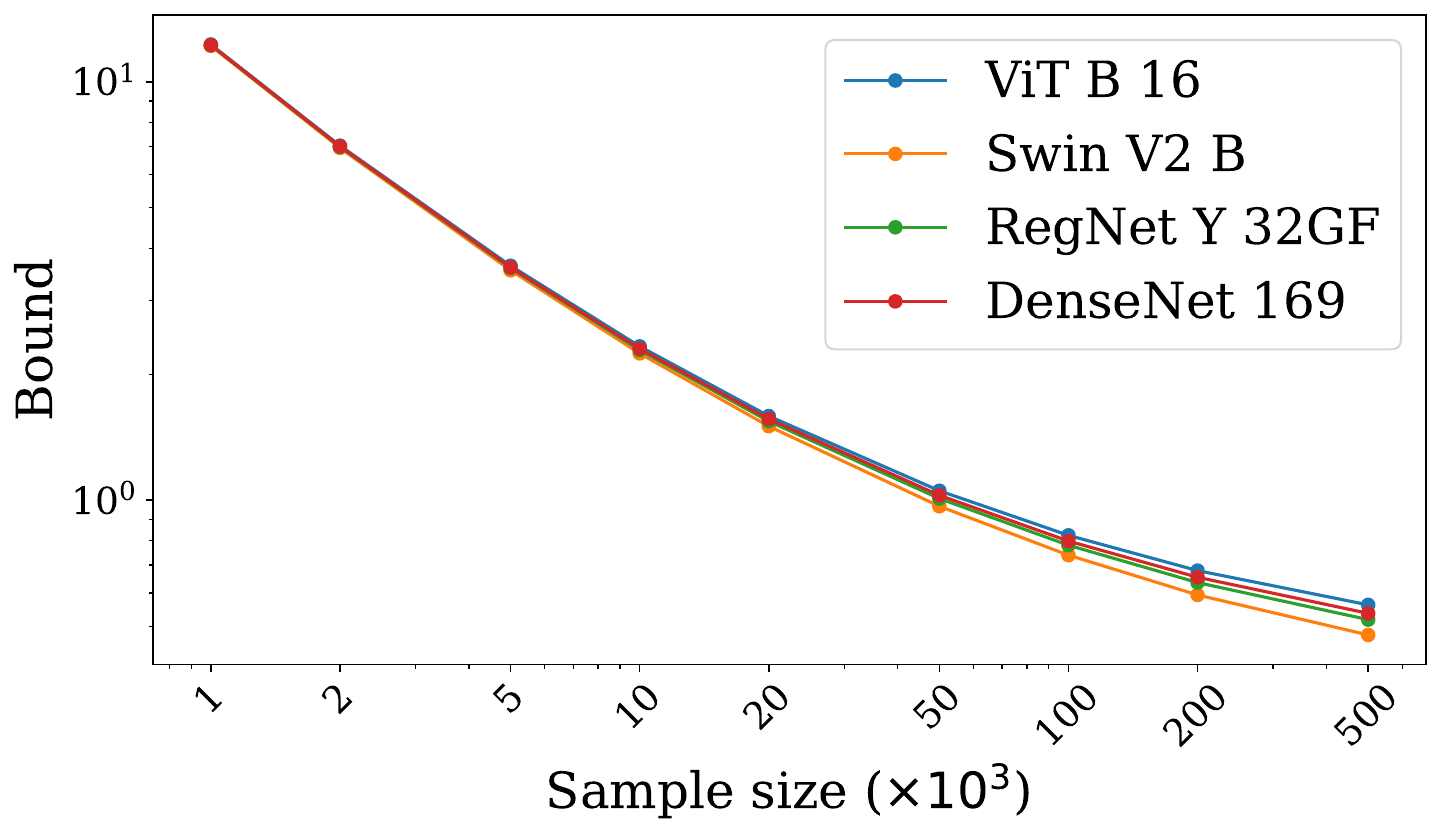}
                \caption{Bound (\ref{thm-gen-train-small-K-any-distribution-eq}) as $n$ increases.} 
                \label{fig-bound-sample-size}
        \end{minipage}
 \end{figure}

\textbf{Sample size:} Figure~\ref{fig-bound-sample-size} illustrates the dependence of our bound on the sample size $n$. As $n$ increases, the bound contracts rapidly. Although it remains vacuous in low-data regimes, Bound~(\ref{thm-gen-train-small-K-any-distribution-eq}) becomes non-vacuous once the sample size exceeds roughly ten thousand. This behavior highlights a limitation of our approach when data are scarce, where certain classical bounds \citep{nadjahi24slicingMI,biggs2022nonvacuousBound,dong2025exactly} may provide tighter guarantees. However, in data-rich regimes, our bound tightens quickly and ultimately outperforms these alternatives, demonstrating its effectiveness for modern large-scale settings.

\textbf{Other ablations:} We also did various ablation studies, including (i) Impact of the data-partition alignment (Appendix~\ref{app-data-partition-alignment}), (ii) Correlation between test error and our overall bounds (Appendix~\ref{app-Correlation-test-error-bounds}), (iii) Cost for computing the bound (Appendix~\ref{app-Computational-complexity}). Due to page limit, we present those in appendices.

\section{Conclusion and Discussion} \label{sec-Conclusion}

Understanding and certifying the behavior of modern deep networks remains a foundational challenge for reliable machine learning. This work introduces a new class of \textit{data-dependent generalization bounds} that apply directly to trained models, without compression, architectural modification, or retraining. Among them, the exactly computable bound stands out: it is non-vacuous across all evaluated ImageNet-scale models. This closes part of a longstanding gap in learning theory. 

A central insight of our framework is that generalization can be decomposed into two interpretable components: (1) a \textit{distributional complexity term}, capturing how concentrated or diffuse the data distribution is across the partition; and (2) \textit{local model-behavior term}, capturing how the trained network behaves in specific regions of the input space.
This joint dependence does not appear in classical bounds. It reveals where the generalization gap arises and why, highlighting the specific regions and local behaviors responsible. Empirically, we observe that components of Bound~(\ref{thm-gen-train-small-K-eq}), such as the local-loss–weighted concentration term, are highly predictive of the true test error. 

At the same time, our framework has some limitations that point to fertile directions for future research. The bounds can be loose when partitions are poorly aligned with the data distribution or when sample sizes are small, since multinomial counts may be far from their expectations. Likewise, when the model underfits the data, the empirical loss dominates the bound. These failure modes clarify that the approach naturally excels in high-data, high-capacity regimes  while classical bounds may remain more effective for small datasets or highly structured models.

\input{Nonvacuous-DNN.bbl}
\newpage
\appendix
\onecolumn

\input{DL-bound-sup}

\end{document}

%% file: math_commands.tex
\usepackage{amsmath,amsfonts,bm}

\def\eqref#1{equation~\ref{#1}}

\def\1{\bm{1}}

\def\vh{{\bm{h}}}

\def\vn{{\bm{n}}}

\def\vv{{\bm{v}}}

\def\vx{{\bm{x}}}
\def\vy{{\bm{y}}}
\def\vz{{\bm{z}}}

\def\mS{{\bm{S}}}
\def\mT{{\bm{T}}}

\def\mY{{\bm{Y}}}

\DeclareMathAlphabet{\mathsfit}{\encodingdefault}{\sfdefault}{m}{sl}
\SetMathAlphabet{\mathsfit}{bold}{\encodingdefault}{\sfdefault}{bx}{n}

\def\gX{{\mathcal{X}}}
\def\gY{{\mathcal{Y}}}
\def\gZ{{\mathcal{Z}}}

\newcommand{\E}{\mathbb{E}}

\newcommand{\R}{\mathbb{R}}



%% file: DL-bound-sup.tex
\section{Proofs for main results}\label{app-Proofs-main-results}

\begin{proof}[Proof of Theorem \ref{thm-gen-train-small-K}]
We first observe that
\begin{eqnarray}
\label{eq-app-thm-gen-train-small-K-1}
F(P,\vh)- F(\mS,\vh) =  F(P,\vh) - \sum_{i=1}^K \frac{n_i}{n} a_i(\vh) + \sum_{i=1}^K \frac{n_i}{n} a_i(\vh) - F(\mS,\vh)
\end{eqnarray}

Next, we consider $F(P, \vh) -  \sum_{i=1}^K \frac{n_i}{n} a_i(\vh) =  \sum_{i=1}^K p_i a_i(\vh)  -  \sum_{i=1}^K \frac{n_i}{n} a_i(\vh) =  \sum_{i=1}^K a_i(\vh) \left[p_i  -\frac{n_i}{n}\right]$.  Note that $(n_1, ..., n_K)$ is a multinomial random variable with parameters $n$ and $(p_1, ..., p_K)$. Therefore, according to Lemma~7 in \citep{kawaguchi22RobustGen}, we have 
$\Pr\left(\sum_{i=1}^{K } a_i(\vh) \left[p_i -  \frac{n_i}{n} \right]   > g(\Gamma,\vh,\delta_2)  \right) < \delta_2 $.
This implies 
\begin{eqnarray}
\label{eq-app-thm-gen-train-small-K-2}
\Pr\left(F(P, \vh) -  \sum_{i=1}^K \frac{n_i}{n} a_i(\vh)   > g(\Gamma,\vh,\delta_2)  \right) < \delta_2
\end{eqnarray}

On the other hand, Theorem \ref{thm-gen-general-data-part} below shows that
\begin{eqnarray}
\Pr\left(  \sum_{i \in \mT} \frac{n_i}{n} a_i(\vh) - F(\mS,\vh)  \ge C\sqrt{\frac{u}{2n^2} \ln \frac{1}{\delta_1} } \right) &\le&  \delta_1
\end{eqnarray}
Combining this with (\ref{eq-app-thm-gen-train-small-K-2}) and the union bound, we have
\begin{eqnarray}
\Pr\left(F(P, \vh)  > F(\mS,\vh)  + C\sqrt{\frac{u}{2n^2} \ln\frac{1}{\delta_1} } + g(\Gamma,\vh,\delta_2) \right)  < \delta_1 + \delta_2
\end{eqnarray}
 completing the proof.
\end{proof}

\begin{proof}[Proof of Corollary \ref{cor-gen-train-small-K-continuous}]
A simple consequence of using quantiles for continuous distributions \citep{hallin2021distribution,figalli2018continuity} suggests that there exists a partition $\Gamma^*(\gZ) := \bigcup_{i=1}^{K } \gZ_i^*$ so that $P(\gZ_i^*) = \frac{1}{K}, \forall i \in [K]$. The  result of this corollary can be derived by applying Theorem~\ref{thm-gen-train-small-K} for paritition $\Gamma^*$, where $p_i = 1/K$ for all $i$.
\end{proof}

\begin{proof}[Proof of Theorem \ref{thm-gen-train-small-K-any-distribution}]

Theorem \ref{thm-gen-train-small-K} shows that 
\begin{eqnarray}
\label{thm-gen-train-small-K-any-distribution-eq-0}
\Pr\left(F(P, \vh)  > F(\mS,\vh)  + C\sqrt{\frac{u}{2n^2} \ln\frac{1}{\delta_1} } + g(\Gamma,\vh,\delta/2) \right)  < \delta_1 + \delta/2
\end{eqnarray}
where $u$ and $\delta_1$ depend on the sum $\sum_{i=1}^K p_i^2$. We next bound this quantity using $\mS$.

Since $p_i \ge 0$ and $\sum_{i=1}^K p_i = 1$, we can  use the Lagrange multiplier method to show that $\sum_{i=1}^K p_i^2$ is minimized at $1/K$. Hence $u = \sum_{i=1}^K \gamma n p_i(1+ \gamma  n p_i) = \gamma n + \gamma^2 n^2 \sum_{i=1}^K p_i^2 \ge \gamma n + \gamma^2 n^2 /K$. This suggests that $\exp(- \frac{ u \ln\gamma }{4n-3}) \le \exp(- \frac{(\gamma n + \gamma^2 n^2 /K) \ln\gamma }{4n-3}) \le \exp(- \frac{\gamma n(K + \gamma n) \ln\gamma }{K(4n-3)}) \le \gamma^{-\alpha}$. Choosing $\delta_1 = \gamma^{-\alpha}$ and plugging it into (\ref{thm-gen-train-small-K-any-distribution-eq-0}) lead to
\begin{equation}
\Pr\left(F(P, \vh) > F(\mS,\vh)  + C\sqrt{\frac{u}{2n^2}\alpha \ln\gamma  } + g(\Gamma,\vh,\delta/2)  \right) < \delta/2 + \gamma^{-\alpha}
\end{equation}
It is easy to see that $g(\Gamma,\vh,\delta/2) \le g_2(\delta/2)$, since $a_o(\vh) \le C$ and $a_i(\vh) \le C$ for any $i$. Therefore
\begin{equation}
\label{thm-gen-train-small-K-any-distribution-eq-1}
\Pr\left(F(P, \vh) > F(\mS,\vh)  + C\sqrt{\frac{u}{2n^2}\alpha \ln\gamma  } + g_2(\delta/2) \right) < \delta/2 + \gamma^{-\alpha}
\end{equation}

Next we consider  $\frac{u}{2n^2} = \frac{\gamma }{2n} +  \frac{\gamma^2 }{2}\sum_{i=1}^K p_i^2$. Since $\mS$ contains $n$ i.i.d. samples, $(n_1, ..., n_K)$ is a multinomial random variable with parameters $n$ and $(p_1, ..., p_K)$.  Lemma~\ref{lem-multinomial-square} shows 
\[\Pr\left( \sum_{i=1}^K p_i^2  > \sum_{i=1}^K \left( \frac{n_i}{n} \right)^2 + 2 \sqrt{\frac{2}{n}\ln\frac{2K}{\delta}} \right) < \delta/2\] 
Therefore $\Pr\left( \frac{u}{2n^2} > \frac{\gamma}{2n} +  \frac{\gamma^2 }{2}\sum_{i=1}^K \left( \frac{n_i}{n} \right)^2 + \gamma^2 \sqrt{\frac{2}{n}\ln\frac{2K}{\delta}} \right) < \delta/2$. This also suggests that 
\begin{eqnarray}
 \label{thm-gen-train-small-K-any-distribution-eq-4}
\Pr\left( C\sqrt{\frac{u}{2n^2}\alpha \ln\gamma  }  > C\sqrt{\hat{u}\alpha \ln\gamma  }  \right) < \delta/2
 \end{eqnarray}
Combining this with (\ref{thm-gen-train-small-K-any-distribution-eq-1}) and the union bound will complete the proof.
\end{proof}

\subsection{Approximating the intractable part by a data set}

\begin{theorem} \label{thm-gen-general-data-part}
Given the notations in Theorem \ref{thm-gen-train-small-K},
\begin{eqnarray}
 \label{thm-gen-general-data-part-001}
\Pr\left( \sum_{i \in \mT} \frac{ n_i}{n} a_i(\vh) \ge \sum_{i \in \mT} \frac{n_i}{n} F(\mS_i, \vh)  + C\sqrt{\frac{u}{2n^2} \ln \frac{1}{\delta_1} } \right) &\le&  \delta_1
\end{eqnarray}
\end{theorem}

\begin{proof}  
Denote $\vn = \{n_1, ..., n_K\}$ and for each $j \in [K]$ and function $\vv$: 
\begin{eqnarray}
B_j(\vv) &=& \sum_{i=1}^j  n_i a_i(\vv) - \sum_{i=1}^j n_i F(\mS_i, \vv) \\
X_j(\vv) &=& n_j F(\mS_j, \vv) \\
\mS_{\le j} &=& \bigcup_{i \le j} \mS_i
\end{eqnarray}
Denote $y = \frac{4t}{u C^2}$ for any $t \in \left[0, uC\sqrt{\frac{\ln\gamma}{8n-6}}\right]$. The proof for (\ref{thm-gen-general-data-part-001}) contains three main steps.

\textbf{Step 1:} We first observe that
\begin{align}
\Pr\left(B_K(\vh) \ge t \right) &\le e^{-y t} \E_{\mS}\left[e^{y B_K(\vh)}\right] & (\text{Chernoff bounds})\\
 \label{thm-gen-general-data-part-01}
 &\le  e^{-y t}  \E_{\vv,\vn}\left[ \E_{\mS}\left[e^{y B_K(\vv)} | \vv, \vn\right] \right] & (\text{Law of total expectation})
\end{align}

\textbf{Step 2 -  estimating $\E_{\mS}\left[e^{y B_K(\vv)} | \vv, \vn\right]$:} We observe the following  for each $j \in \mT$,
\begin{align}
\E_{X_j} [X_j | \vv,\vn] &= \E_{\mS_j} [n_j F(\mS_j, \vv) | \vv,\vn] \\
&=  \E_{\mS_j} \left[ \sum_{i=1}^{n_j} \ell(\vv,\vz_{ji}) | \vv,\vn \right] & (\text{where } \mS_j = \{\vz_{ji}\}_{i=1}^{n_j})\\
&=  \sum_{i=1}^{n_j} \E_{\vz_{ji} \in \gZ_j} \left[\ell(\vv,\vz_{ji}) | \vv,\vn \right] & (\mS_j \text{ contains i.i.d. samples in } \gZ_j)\\
&= \sum_{i=1}^{n_j}  a_j(\vv) = n_j a_j(\vv)
\end{align}
Therefore $B_j = B_{j-1} + \E_{X_j} [X_j | \vv,\vn] - X_j$ for all  $j \in \mT$. Note that $B_i = B_{i-1}$ (due to $n_i = b_i = X_i = 0$) for all $ i \notin \mT$. Hence, for $ i \notin \mT$, we will use $\E_{X_i} [X_i | \vv,\vn] - X_i$ instead of 0 in the below analysis  for simplicity of presentation. 

We can rewrite
\begin{align}
\E_{\mS}\left[e^{y B_K(\vv)} | \vv, \vn\right]
 &= \E_{\mS}\left[e^{y (B_{K-1} + \E_{X_K} [X_K | \vv,\vn]  - X_K )} | \vv,\vn\right]\\
 &= \E_{\mS_{\le K}}\left[e^{y (B_{K-1} + \E_{X_K} [X_K | \vv,\vn] - X_K )} | \vv,\vn\right]\\
  \label{thm-gen-general-data-part-02}
 &\le \E_{\mS_{\le K-1}}\left[e^{y B_{K-1}}  | \vv,\vn \right]  \E_{X_K}\left[e^{ y( \E_{X_K} [X_K | \vv,\vn] - X_K )} | \vv,\vn \right] 
\end{align}
where the last inequality comes from the fact that $X_K$ is conditionally independent with $\mS_{\le K-1}$, conditioned on $\{\vv,\vn\}$. 

It is easy to see that $0 \le X_K \le  Cn_K$, due to $0 \le F(\mS_K, \vv) \le C$. Lemma~\ref{lem-Hoeffding-lemma} implies $\E_{X_K}\left[e^{ y( \E_{X_K} [X_K | \vv,\vn] - X_K )} | \vv,\vn \right] \le \exp\left(\frac{y^2 C^2 n_K^2}{8} \right)$. Plugging this into (\ref{thm-gen-general-data-part-02}), we obtain
\begin{align}
  \label{thm-gen-general-data-part-03}
\E_{\mS}\left[e^{y B_K} | \vv, \vn\right]
 &\le \E_{\mS_{\le K-1}}\left[e^{y B_{K-1}}  | \vv,\vn \right]  \exp\left(\frac{y^2 C^2 n_K^2}{8} \right)
\end{align}
Using the same arguments for $X_{K-1}, ..., X_1$, we obtain the followings
\begin{eqnarray}
\nonumber
\E_{\mS}\left[e^{y B_K(\vv)} | \vv, \vn\right]
 &\le& \E_{\mS_{\le K-2}}\left[e^{y B_{K-2}}  | \vv,\vn \right]  \exp\left(\frac{y^2 C^2 n_K^2}{8} + \frac{y^2 C^2 n_{K-1}^2}{8} \right) \\
 \nonumber
 & ...& \\
   \label{thm-gen-general-data-part-04}
 &\le&  \exp\left(\frac{y^2 C^2}{8} \sum_{i=1}^{K}  n_i^2 \right)
\end{eqnarray}

\textbf{Step 3 - bounding $\Pr\left(B_K(\vh) \ge t \right) $:} By combining (\ref{thm-gen-general-data-part-04}) with (\ref{thm-gen-general-data-part-01}), we obtain
\begin{eqnarray}
\Pr\left(B_K(\vh) \ge t \right) 
&\le& e^{-y t} \E_{\vv,\vn} \exp\left(\frac{y^2C^2}{8} \sum_{i=1}^{K} n_i^2 \right) \\
\label{thm-gen-general-data-part-05}
&\le& e^{-y t} \E_{\vn} \exp\left(\frac{y^2C^2}{8} \sum_{i=1}^{K} n_i^2 \right) 
\end{eqnarray}

It is well known that multinomial counts are negatively associated. For such random variables, the following exponential inequality holds \citep{joagdev1983negative,dubhashi1998balls}:
if $f_i$ are coordinate-wise nondecreasing functions, then
\[
\E\prod_{i=1}^K f_i(n_i)
\le
\prod_{i=1}^K \E f_i(n_i).
\]

Now take $f_i(n_i)=\exp(\frac{y^2C^2}{8} n_i^2)$ which is increasing in $n_i$ on $[0,n]$. Therefore,
\[
\E_{\vn} e^{\frac{y^2C^2}{8}  \sum_{i=1}^K n_i^2}
= \E_{\vn} \prod_{i=1}^K e^{\frac{y^2C^2}{8}  n_i^2}
\le \prod_{i=1}^K \E e^{\frac{y^2C^2}{8} n_i^2}.
\]

Plugging this into (\ref{thm-gen-general-data-part-05}), we have 
\begin{eqnarray}
\Pr\left(B_K(\vh) \ge t \right) 
&\le& e^{-y t} \E_{\vn} \exp\left(\frac{y^2C^2}{8} \sum_{i=1}^{K} n_i^2 \right) \\
\label{thm-gen-general-data-part-05.2}
&\le& e^{-y t} \prod_{i=1}^K \E \exp\left(\frac{y^2 C^2}{8} n_i^2 \right). 
\end{eqnarray}


When $\gamma p_i < 1$, due to $t \le uC\sqrt{\frac{\ln\gamma}{8n-6}}$, observe that $\frac{y^2C^2}{8} = \frac{2t^2 }{u^2 C^2} \le \frac{\ln\gamma}{4n-3}  \le \frac{\ln\gamma}{(1-\gamma p_i)(4n-3)}$. Note that $n_i$ is a binomial random variable with parameters $n$ and $p_i$. Combining those facts with  Lemma \ref{lem-binomial-exp-square} implies $\E_{n_i} \exp\left(\frac{y^2 C^2}{8} n_i^2 \right) \le  \exp\left(\frac{y^2 C^2 }{8} \gamma n p_i(1 + \gamma n p_i) \right)$. On the other hand, Lemma  \ref{lem-binomial-exp-square-large-mean} also implies $\E_{n_i} \exp\left(\frac{y^2 C^2}{8} n_i^2 \right) \le  \exp\left(\frac{y^2 C^2 }{8} \gamma n p_i(1 + \gamma n p_i) \right)$ when  $\gamma p_i \ge 1$. 
As a result, those facts and (\ref{thm-gen-general-data-part-05.2}) lead to the following: 
\begin{eqnarray}
\Pr\left(B_K(\vh) \ge t \right)  &\le& \exp\left(-yt + \frac{y^2 C^2}{8} \sum_{i=1}^{K}  (1  + \gamma n p_i) \gamma n p_i   \right) \\  
&=& \exp\left(-yt + \frac{y^2 C^2 u}{8}  \right) = \exp\left( \frac{ -2 t^2}{uC^2} \right)
\end{eqnarray}
As a result 
\begin{eqnarray}
\Pr\left( \sum_{i =1}^K n_i a_i(\vh) \ge \sum_{i =1}^K n_i F(\mS_i, \vh)  + t \right) &\le&  \exp\left(-\frac{2 t^2}{uC^2}\right) 
\end{eqnarray}
Since $n_j = 0$ for all $j \notin \mT$, we have 
\begin{eqnarray}
\Pr\left( \sum_{i \in \mT} n_i a_i(\vh) \ge \sum_{i \in \mT} n_i F(\mS_i, \vh)  + t \right) &\le&  \exp\left(-\frac{2 t^2}{uC^2}\right) 
\end{eqnarray}
Multiplying both sides (of the probability term) with $1/n$ leads to
\[\Pr\left( \sum_{i \in \mT} \frac{ n_i}{n} a_i(\vh) \ge \sum_{i \in \mT} \frac{ n_i}{n} F(\mS_i, \vh)  + t/n \right) \le  \exp\left(-\frac{2 t^2}{uC^2}\right) \]
Choosing $t = C\sqrt{\frac{u}{2} \ln \frac{1}{\delta_1} }$ results in (\ref{thm-gen-general-data-part-001}), completing the proof.
\end{proof}

\section{Supporting theorems and lemmas}\label{app-Supporting-theorems}

\subsection{Hoeffding's Lemma} \label{sec-lem-Hoeffding-lemma}

\begin{lemma}[Hoeffding's lemma for conditionals] \label{lem-Hoeffding-lemma}
Let $X$ be any real-valued random variable that may depend on some random variables $\mY$. Assume that  $a \le X \le b$ almost surely, for some constants $a,b$. Then, for all $\lambda \in \R$,
\begin{eqnarray}
\label{eq-lem-Hoeffding-lemma-01}
\E_X \left[e^{\lambda (\E_X [X| \mY] - X)} | \mY \right] &\le& \exp\left(\frac{\lambda^2 (b-a)^2}{8} \right) 
\end{eqnarray}
\end{lemma}

\begin{proof}  
Denote $c = \E_X [X| \mY] - b, d = \E_X [X| \mY] - a$ and hence $c \le 0 \le d$.

Since $\exp$ is a convex function, we have the following for all $\E_X [X| \mY] - X \in [c,d]$:

\[ e^{\lambda (\E_X [X| \mY] -X)} \le \frac{d- \E_X [X| \mY] +X}{d-c} e^{\lambda c} + \frac{\E_X [X| \mY] -X -c}{d-c} e^{\lambda d} \]
Therefore, by taking the conditional expectation over $X$ for both sides,
\begin{eqnarray}
\nonumber
\E_X \left[e^{\lambda (\E_X [X| \mY] -X)} | \mY \right] &\le&  \frac{d- \E_X [X| \mY] +\E_X [X| \mY]}{d-c} e^{\lambda c} + \frac{\E_X [X| \mY] -\E_X [X| \mY] -c}{d-c} e^{\lambda d} \\
&=&  \frac{d}{d-c} e^{\lambda c} - \frac{c}{d-c} e^{\lambda d} \\
\label{eq-app-Hoeffding-lemma}
&=&  e^{L(\lambda(d-c))} 
\end{eqnarray}
where $L(h) = \frac{c h}{d-c} + \ln(1 + \frac{c - e^h c}{d-c})$. For this function, note that 

\[L(0) = L'(0) = 0 \text{ and } L''(h) = -\frac{cd e^h}{(d - c e^h)^2} \]

The AM-GM  inequality suggests that $L''(h) \le 1/4$ for all $h$. Combining this property with Taylor's theorem leads to the following, for some $\theta \in [0,1]$, 

\[L(h) = L(0) + h L'(0) + \frac{1}{2} h^2 L''(h\theta) \le \frac{h^2}{8} \]
Combining this with (\ref{eq-app-Hoeffding-lemma}) completes the proof.
\end{proof}

\subsection{Small random variables} \label{app-Small-random-variables}

\begin{lemma}\label{lem-small-exp-square-large-mean}
Let $x_1, ..., x_n$ be independent  random variables in $[0,1]$ and satisfy $\E[x_i] \le \nu, \forall i$ for some $\nu \in [0,1]$. For any $c \ge 1$ satisfying $c \nu \ge 1$ and any  $\lambda \ge 0$, we have $\E \exp \left(\lambda (x_1 + \cdots + x_n)^2\right) \le \exp({\lambda c n\nu(1 + c n\nu)})$.
\end{lemma}

\begin{lemma}\label{lem-small-exp-square}
Let $x_1, ..., x_n$ be independent random variables in $[0,1]$ and satisfy $\E[x_i] \le \nu, \forall i$ for some $\nu \in [0,1]$. For any $c \ge 1$ satisfying $c \nu < 1$ and any  $\lambda \in [0, \frac{\ln c}{(1-c\nu)(4n-3)}]$, we have $\E \exp \left(\lambda (x_1 + \cdots + x_n)^2\right) \le \exp({\lambda c n\nu(1 + c n\nu)})$.
\end{lemma}

In order to prove those results, we need the following observations.

\begin{lemma}\label{lem-small-random-variable}
Consider  a  random variable $X \in [0,1]$ with mean $\E[X] \le \nu$ for some constant $\nu \in [0,1]$. For any $c \ge 1, \lambda \ge 0$: 
\begin{itemize}
\item If $c \nu \ge 1$, then $\E e^{\lambda X} \le e^{c\nu \lambda}$.
\item If $c \nu < 1$, then $\E e^{\lambda X} \le e^{c\nu \lambda}$ for all $\lambda \in [0, \frac{\ln c}{1-c\nu}]$.
\end{itemize}
\end{lemma}

\begin{proof}
The Taylor series expansion of the function $e^{\lambda X}$ at any $X$ is $e^{\lambda X} = 1 + \sum_{p=1}^{\infty} \frac{(\lambda X)^p}{p!}$. Therefore
\begin{eqnarray}
\E [e^{\lambda X}]  &=& 1 + \sum_{p=1}^{\infty} \frac{\lambda^p}{p!} \E(X^p) \le 1 + \E(X) \sum_{p=1}^{\infty} \frac{\lambda^p}{p!}  \;\;\;\;\;\;\;\;\; (\text{due to } X^p \le X, \forall p \ge 1) \\
&\le& 1 + \nu \sum_{p=1}^{\infty} \frac{\lambda^p}{p!} = 1 + \nu (e^{\lambda} -1) =  1 -\nu + \nu e^{\lambda}
\end{eqnarray}

Next we consider function $y(\lambda) = e^{c\nu\lambda} - 1 +\nu - \nu e^{\lambda}$. Its derivative is $y' = c\nu e^{c\nu\lambda} - \nu e^{\lambda} = \nu e^{\lambda} ( c e^{(c\nu-1)\lambda} -1)$. 

For the case $c\nu \ge 1$, one can observe that $y' \ge 0$ for all $\lambda \ge 0$. This means $y$ is non-decreasing, and hence $y(\lambda) \ge y(0) = 0$. As a result, $e^{c\nu\lambda} \ge 1 -\nu + \nu e^{\lambda} \ge \E [e^{\lambda X}]$.

Consider the case $c\nu < 1$, it is easy to show that $y'(\lambda) \ge 0$ for all $\lambda \in [0, \frac{\ln c}{1-c\nu}]$. This means $y$ is non-decreasing in the interval $[0, \frac{\ln c}{1-c\nu}]$, and hence $y(\lambda) \ge y(0) = 0$ for all $\lambda \in [0, \frac{\ln c}{1-c\nu}]$. As a result, $e^{c\nu\lambda} \ge 1 -\nu + \nu e^{\lambda} \ge \E [e^{\lambda X}]$, completing the proof.
\end{proof}

\begin{corollary}\label{cor-small-random-variable-exp-mix}
Consider  a  random variable $X \in [0,1]$ with mean $\E[X] \le \nu$ for some constant $\nu \in [0,1]$. For all constants $a,b \ge 0, c \ge 1$:
\begin{itemize}
\item $\E e^{\lambda (a X^2 + b X)} \le e^{c(a+b)\nu\lambda}$, for all $\lambda \ge 0$, if $c \nu \ge 1$.
\item $\E e^{\lambda (a X^2 + b X)} \le e^{c(a+b)\nu\lambda}$, for all $\lambda \in [0, \frac{\ln c}{(1-c\nu)(a+b)}]$, if $c \nu < 1$.
\end{itemize}
\end{corollary}

\begin{proof}
It is easy to observe that $\E e^{\lambda (a X^2)} \le \E e^{\lambda (a X)}$ due to $X \in [0,1]$. This suggests that $\E e^{\lambda (a X^2 + b X)} \le \E e^{\lambda (a + b) X}$. Applying Lemma~\ref{lem-small-random-variable} will complete the proof.
\end{proof}

\begin{proof}[Proof of Lemma \ref{lem-small-exp-square-large-mean}]
Denote $y_n = x_1 + \cdots + x_n$. Observe that $y_{n} = y_{n-1} + x_{n}$ and
\begin{eqnarray}
\label{lem-binomial-exp-square-large-mean-eq-1}
\E_{y_n} e^{\lambda y_n^2} &=& \E_{y_n} e^{\lambda (y_{n-1}^2 + 2x_n y_{n-1} + x_n^2)}  
= \E_{y_{n-1}} \left[e^{\lambda y_{n-1}^2} \E_{x_{n}} e^{\lambda (2x_n y_{n-1} + x_n^2)} \right] 
\end{eqnarray}
Since $c \nu \ge 1$ and  $x_n$ is independent with $y_{n-1}$, Corollary~\ref{cor-small-random-variable-exp-mix} implies $\E_{x_{n}} e^{\lambda (2x_n y_{n-1} + x_n^2)} \le e^{c\nu\lambda ( 2y_{n-1} +1)}$. Plugging this into (\ref{lem-binomial-exp-square-large-mean-eq-1}) leads to
\begin{eqnarray}
\label{lem-binomial-exp-square-large-mean-eq-1.1}
\E_{y_n} e^{\lambda y_n^2} &\le&  \E_{y_{n-1}} \left[e^{\lambda y_{n-1}^2} e^{c\nu\lambda (2y_{n-1} +1)} \right] 
= e^{c\nu\lambda} \E_{y_{n-1}} \left[e^{\lambda (y_{n-1}^2+ 2c\nu y_{n-1})} \right]
\end{eqnarray}
 
Next we consider $\E_{y_{n-1}} \left[e^{\lambda (y_{n-1}^2+ 2c\nu y_{n-1})} \right]$. Observe that $y_{n-1} = y_{n-2} + x_{n-1}$ and hence

\begin{eqnarray}
\E_{y_{n-1}} \left[e^{\lambda (y_{n-1}^2+ 2c\nu y_{n-1})} \right] &=& \E_{y_{n-1}} e^{\lambda (y_{n-2}^2 + 2x_{n-1} y_{n-2} + x_{n-1}^2 + 2c\nu x_{n-1} + 2c\nu y_{n-2})}  \\
\label{lem-binomial-exp-square-large-mean-eq-2}
&=& \E_{y_{n-2}} \left[e^{\lambda (y_{n-2}^2 + 2c\nu y_{n-2})} \E_{x_{n-1}} e^{\lambda (2x_{n-1} y_{n-2} +  2c\nu x_{n-1}  + x_{n-1}^2)} \right] 
\end{eqnarray}
Since $c \nu \ge 1$ and  $x_{n-1}$ is independent with $y_{n-2}$, Corollary~\ref{cor-small-random-variable-exp-mix} implies $\E_{x_{n-1}} e^{\lambda (2x_{n-1} y_{n-2} +  2c\nu x_{n-1}  + x_{n-1}^2)} \le e^{c\nu\lambda (2y_{n-2} +  2c\nu  + 1)}$. Plugging this into (\ref{lem-binomial-exp-square-large-mean-eq-2}) leads to

\begin{eqnarray}
\E_{y_{n-1}} \left[e^{\lambda (y_{n-1}^2+ 2c\nu y_{n-1})} \right] 
&\le& \E_{y_{n-2}} \left[e^{\lambda (y_{n-2}^2 + 2c\nu y_{n-2})}  e^{c\nu\lambda (2y_{n-2} +  2c\nu  + 1)} \right] \\
&=& e^{c\nu\lambda(2c\nu  + 1)} \E_{y_{n-2}} \left[e^{\lambda (y_{n-2}^2+ 4c\nu y_{n-2})} \right]
\end{eqnarray}
By using the same arguments, we can show that 
\begin{eqnarray}
\E_{y_{n-1}} \left[e^{\lambda (y_{n-1}^2+ 2c\nu y_{n-1})} \right] 
&\le&  e^{c\nu\lambda(2c\nu  + 1)} e^{c\nu\lambda(4c\nu  + 1)} \E_{y_{n-3}} \left[e^{\lambda (y_{n-3}^2+ 6c\nu y_{n-3})} \right] \\
&=&  e^{2c\nu\lambda(3c\nu  + 1)} \E_{y_{n-3}} \left[e^{\lambda (y_{n-3}^2+ 6c\nu y_{n-3})} \right] \\
\nonumber
&...& \\
\label{lem-binomial-exp-square-large-mean-eq-3}
&\le&  e^{c(n-2)\nu\lambda(c(n-1)\nu  + 1)} \E_{y_{1}} \left[e^{\lambda (y_{1}^2+ 2c(n-1)\nu y_{1})} \right]
\end{eqnarray}
Note that $\E_{y_{1}} \left[e^{\lambda (y_{1}^2+ 2c(n-1)\nu y_{1})} \right] = \E_{x_{1}} \left[e^{\lambda (x_{1}^2+ 2c(n-1)\nu x_{1})} \right] \le  e^{c\nu\lambda(1 + 2c(n-1)\nu)}$, according to Corollary~\ref{cor-small-random-variable-exp-mix}. Combining this with (\ref{lem-binomial-exp-square-large-mean-eq-3}), we obtain
\begin{eqnarray}
\label{lem-binomial-exp-square-large-mean-eq-4}
\E_{y_{n-1}} \left[e^{\lambda (y_{n-1}^2+ 2c\nu y_{n-1})} \right] 
&\le&  e^{c(n-2)\nu\lambda(c(n-1)\nu  + 1)} e^{c\nu\lambda(1 + 2c(n-1)\nu)} = e^{c\nu\lambda(1 + cn\nu)(n-1)} 
\end{eqnarray}
By plugging this into (\ref{lem-binomial-exp-square-large-mean-eq-1.1}), we obtain
\begin{eqnarray}
\E_{y_n} e^{\lambda y_n^2} &\le&   e^{c\nu\lambda} e^{c\nu\lambda(1 + cn\nu)(n-1)} = e^{c\nu\lambda((1 + cn\nu)n - cn\nu)} \\
 &\le& e^{cn\nu(1 + c n\nu)\lambda}
\end{eqnarray}
completing the proof.
\end{proof}

\begin{proof}[Proof of Lemma \ref{lem-small-exp-square}]
Denote $y_n = x_1 + \cdots + x_n$ and observe that
\begin{eqnarray}
\label{lem-binomial-exp-square-eq-1}
\E_{y_n} e^{\lambda y_n^2} &=& \E_{y_n} e^{\lambda (y_{n-1}^2 + 2x_n y_{n-1} + x_n^2)}  
= \E_{y_{n-1}} \left[e^{\lambda y_{n-1}^2} \E_{x_{n}} e^{\lambda (2x_n y_{n-1} + x_n^2)} \right] 
\end{eqnarray}
Note that $y_{n-1} = x_1 + \cdots + x_{n-1} \le n-1$ and $\lambda (2y_{n-1} +1) \le \lambda(2n-1) \le \lambda(4n-3) \le \frac{\ln c}{1-c\nu}$. Since   $x_{n}$ is independent with $y_{n-1}$, Corollary~\ref{cor-small-random-variable-exp-mix} implies $\E_{x_{n}} e^{\lambda (2x_n y_{n-1} + x_n^2)} \le e^{c\nu\lambda ( 2y_{n-1} +1)}$. Plugging this into (\ref{lem-binomial-exp-square-eq-1}) leads to
\begin{eqnarray}
\label{lem-binomial-exp-square-eq-1.1}
\E_{y_n} e^{\lambda y_n^2} &\le&  \E_{y_{n-1}} \left[e^{\lambda y_{n-1}^2} e^{c\nu\lambda (2y_{n-1} +1)} \right] 
= e^{c\nu\lambda} \E_{y_{n-1}} \left[e^{\lambda (y_{n-1}^2+ 2c\nu y_{n-1})} \right]
\end{eqnarray}
 
Next we consider $\E_{y_{n-1}} \left[e^{\lambda (y_{n-1}^2+ 2c\nu y_{n-1})} \right]$. Observe that

\begin{eqnarray}
\E_{y_{n-1}} \left[e^{\lambda (y_{n-1}^2+ 2c\nu y_{n-1})} \right] &=& \E_{y_{n-1}} e^{\lambda (y_{n-2}^2 + 2x_{n-1} y_{n-2} + x_{n-1}^2 + 2c\nu x_{n-1} + 2c\nu y_{n-2})}  \\
\label{lem-binomial-exp-square-eq-2}
&=& \E_{y_{n-2}} \left[e^{\lambda (y_{n-2}^2 + 2c\nu y_{n-2})} \E_{x_{n-1}} e^{\lambda (2x_{n-1} y_{n-2} +  2c\nu x_{n-1}  + x_{n-1}^2)} \right] 
\end{eqnarray}
One can easily show that $\lambda (2 y_{n-2} +  2c\nu +1) \le \lambda (2 ({n-2} ) +  2c\nu +1) \le \lambda(4n-3) \le \frac{\ln c}{1-c\nu}$, since  $y_{n-2} = x_1 + \cdots + x_{n-2} \le n-2$. Therefore Corollary~\ref{cor-small-random-variable-exp-mix} implies $\E_{x_{n-1}} e^{\lambda (2x_{n-1} y_{n-2} +  2c\nu x_{n-1}  + x_{n-1}^2)} \le e^{c\nu\lambda (2y_{n-2} +  2c\nu  + 1)}$, since $x_{n-1}$ is independent with $y_{n-2}$. Plugging this into (\ref{lem-binomial-exp-square-eq-2}) leads to

\begin{eqnarray}
\E_{y_{n-1}} \left[e^{\lambda (y_{n-1}^2+ 2c\nu y_{n-1})} \right] 
&\le& \E_{y_{n-2}} \left[e^{\lambda (y_{n-2}^2 + 2c\nu y_{n-2})}  e^{c\nu\lambda (2y_{n-2} +  2c\nu  + 1)} \right] \\
&=& e^{c\nu\lambda(2c\nu  + 1)} \E_{y_{n-2}} \left[e^{\lambda (y_{n-2}^2+ 4c\nu y_{n-2})} \right]
\end{eqnarray}
By using the same arguments, we can show that 
\begin{eqnarray}
\E_{y_{n-1}} \left[e^{\lambda (y_{n-1}^2+ 2c\nu y_{n-1})} \right] 
&\le&  e^{c\nu\lambda(2c\nu  + 1)} e^{c\nu\lambda(4c\nu  + 1)} \E_{y_{n-3}} \left[e^{\lambda (y_{n-3}^2+ 6c\nu y_{n-3})} \right] \\
&=&  e^{2c\nu\lambda(3c\nu  + 1)} \E_{y_{n-3}} \left[e^{\lambda (y_{n-3}^2+ 6c\nu y_{n-3})} \right] \\
\nonumber
&...& \\
\label{lem-binomial-exp-square-eq-3}
&\le&  e^{c(n-2)\nu\lambda(c(n-1)\nu  + 1)} \E_{y_{1}} \left[e^{\lambda (y_{1}^2+ 2c(n-1)\nu y_{1})} \right]
\end{eqnarray}
Note that $\E_{y_{1}} \left[e^{\lambda (y_{1}^2+ 2c(n-1)\nu y_{1})} \right] = \E_{x_{1}} \left[e^{\lambda (x_{1}^2+ 2c(n-1)\nu x_{1})} \right] \le  e^{c\nu\lambda(1 + 2c(n-1)\nu)}$, according to Corollary~\ref{cor-small-random-variable-exp-mix} and the fact that $\lambda(1 + 2c(n-1)\nu) \le \lambda(4n-3) \le \frac{\ln c}{1-c\nu}$. Combining this with (\ref{lem-binomial-exp-square-eq-3}), we obtain
\begin{eqnarray}
\label{lem-binomial-exp-square-eq-4}
\E_{y_{n-1}} \left[e^{\lambda (y_{n-1}^2+ 2c\nu y_{n-1})} \right] 
&\le&  e^{c(n-2)\nu\lambda(c(n-1)\nu  + 1)} e^{c\nu\lambda(1 + 2c(n-1)\nu)} = e^{c\nu\lambda(1 + cn\nu)(n-1)} 
\end{eqnarray}
By plugging this into (\ref{lem-binomial-exp-square-eq-1.1}), we obtain
\begin{eqnarray}
\E_{y_n} e^{\lambda y_n^2} &\le&   e^{c\nu\lambda} e^{c\nu\lambda(1 + cn\nu)(n-1)} = e^{c\nu\lambda((1 + cn\nu)n - cn\nu)} \\
 &\le& e^{cn\nu(1 + c n\nu)\lambda}
\end{eqnarray}
completing the proof.
\end{proof}

\subsection{Binomial and multinomial random variables} \label{app-Binomial-multinomial}

Next we analyze some properties of  binomial  random variables.

\begin{lemma}\label{lem-binomial-exp-square-large-mean}
Consider a binomial  random variable $z$ with parameters $n \ge 1$ and  $\nu \in [0,1]$. For any $c \ge 1$ satisfying $c \nu \ge 1$ and any  $\lambda \ge 0$, we have $\E e^{\lambda z^2} \le e^{c n\nu(1 + c n\nu)\lambda}$.
\end{lemma}

\begin{proof}
Since $z$ is a binomial  random variable, we can write $z = x_1 + \cdots + x_n$, where $x_1, ..., x_n$ are i.i.d. Bernoulli random variables with parameter $\nu$. Therefore applying Lemma \ref{lem-small-exp-square-large-mean} completes the proof.
\end{proof}

\begin{lemma}\label{lem-binomial-exp-square}
Consider a binomial  random variable $z$ with parameters $n \ge 1$ and $\nu \in [0,1]$. For any $c \ge 1$ satisfying $c \nu < 1$ and any $\lambda \in [0, \frac{\ln c}{(1-c\nu)(4n-3)}]$, we have $\E e^{\lambda z^2} \le e^{c n\nu(1 + c n\nu)\lambda}$.
\end{lemma}

\begin{proof}
Since $z$ is a binomial  random variable, we can write $z = x_1 + \cdots + x_n$, where $x_1, ..., x_n$ are i.i.d. Bernoulli random variables with parameter $\nu$. Therefore applying Lemma \ref{lem-small-exp-square} completes the proof.
\end{proof}

\begin{lemma}[Multinomial variable]\label{lem-multinomial-square}
Consider a multinomial random variable $(n_1, ..., n_K)$ with parameters $n$ and $(p_1, ..., p_K)$. For any $\delta >0$: 
\[ \Pr\left( \sum_{i=1}^K p_i^2  > \sum_{i=1}^K \left( \frac{n_i}{n} \right)^2 + 2 \sqrt{\frac{2}{n}\ln\frac{K}{\delta}} \right) < \delta \]
\end{lemma}

\begin{proof}
Observe that
\begin{eqnarray}
 \sum_{i=1}^K p_i^2 -  \sum_{i=1}^K \left( \frac{n_i}{n} \right)^2 &=&  \sum_{i=1}^K \left[p_i^2 - \left( \frac{n_i}{n} \right)^2 \right] \\
 &=&  \sum_{i=1}^K \left[p_i + \frac{n_i}{n} \right] \left[p_i - \frac{n_i}{n} \right] \\ 
 &=&  2 \sum_{i=1}^K \left(0.5 p_i + \frac{0.5n_i}{n} \right) \left(p_i - \frac{n_i}{n} \right) \\ 
 \label{lem-multinomial-square-eq-1}
 &\le& 2 \max_{i \in [K]} \left(p_i - \frac{n_i}{n} \right)
 \end{eqnarray}
where the last inequlality can be derived by using the fact that $\sum_{i=1}^K \left(0.5 p_i + \frac{0.5n_i}{n} \right) \left(p_i - \frac{n_i}{n} \right)$ is a convex combination of the elements in $\{p_i - \frac{n_i}{n}: i \in [K]\}$, because of $1= \sum_{i=1}^K \left(0.5 p_i + \frac{0.5n_i}{n} \right) $. Furthermore, since $n_i$ is a binomial random variable with parameters $n$ and $p_i$, Lemma~5 in \citep{kawaguchi22RobustGen} shows that $\Pr\left(p_i - \frac{n_i}{n} > \sqrt{\frac{2 p_i}{n}\ln\frac{K}{\delta}} \right) < \delta$ for all $i$. This immediately implies $\Pr\left(p_i - \frac{n_i}{n} > \sqrt{\frac{2}{n}\ln\frac{K}{\delta}} \right) < \delta$. Combining this fact with (\ref{lem-multinomial-square-eq-1}), we obtain
$\Pr\left( \sum_{i=1}^K p_i^2 -  \sum_{i=1}^K \left( \frac{n_i}{n} \right)^2 > 2 \sqrt{\frac{2}{n}\ln\frac{K}{\delta}} \right) < \delta
$, 
completing the proof.
\end{proof}

\newpage
\section{Experimental setup} \label{app-sec-Experimental-setup}

More details about preprocessing and partition:
\begin{itemize}
\item We first preprocessed the images following Pytorch\footnote{\url{ https://pytorch.org/vision/0.20/models/generated/torchvision.models.vit_b_16.html}}: The images are resized to $resize\_size=[256]$ using interpolation=InterpolationMode.BILINEAR, followed by a central crop of $crop\_size=[224]$. Finally the values are first rescaled to $[0.0, 1.0]$. Those operations are required for Pytorch pretrained models.
\item For each run, we randomly choose 200 points in $[0.0,1.0]^{C \times H \times W}$ to be the centroids, since each preprocessed image belongs to $[0.0,1.0]^{C \times H \times W}$. Those centroids are used to build the small areas $\mathcal{Z}_i$ in the partition. Each training image $x$ will be assigned to area $\mathcal{Z}_i$ if it is closest to the centroid of $\mathcal{Z}_i$ amongst all centroids, according to the Euclidean distance.
\end{itemize}

\section{Additional experiment results} \label{app-Additional-experiment-results}

\subsection{Impact of the data-partition alignment}\label{app-data-partition-alignment}

Our bounds can be tighter for a better alignment between the partition and data gemometry. However, analyzing the effect of data geometry and partitioning strategies is  challenging, particularly in high-dimensional settings with unknown distributions. To address this, we designed two controlled ablations  using a synthetic model. 

\textbf{Mixture model (MM):} \textit{each sample $(x, y)$ is generated by
\begin{itemize}
    \item Randomly pick an index $z \sim Cat(\theta)$, a categorical distribution with parameter $\theta = (1/K, \ldots, 1/K) \in \mathbb{R}^K$
    \item Generate $x \sim \mathcal{N}(\mu_z, \nu)$, a normal distribution with mean $\mu_z = (0, \pi * z) \in \mathbb{R}^2$ and variance $\nu$
    \item Return class label $y = 1$ if $z$ is odd, and $y = 0$ otherwise.
\end{itemize}}


\textbf{Exploring different partitioning strategies:} We considered three types of partitions $\Gamma$: 
\textit{\begin{itemize}
    \item \textbf{T1:} A uniform grid partition that divides the data space into equally sized regions. However, this strategy may not align with the actual data distribution, potentially resulting in regions with highly imbalanced probability measures.
    \item \textbf{T2:} A partition formed by uniformly generating $K$ centroids to define the regions. Like T1, this method may not capture the underlying structure of the data.
    \item \textbf{T3:} A partition where the centroids $\mu_1, \ldots, \mu_K$ of the mixture components are fixed as region centers. This approach tends to yield more balanced regions, where $P(\mathcal{Z}_i) \approx P(\mathcal{Z}_j)$ for all $i, j$, for small variances.
\end{itemize}}

Figure \ref{fig-Alignment} visualizes the mixture model (with $\nu=1$) and those three partitions. This figure demonstrates that \textbf{T3} seems to best align with the data geometry, while the other two partitions can be much worse. Figure \ref{fig-Alignment-geometry} illustrates how the data geometry can be changed significantly when varying the variance in the mixture model. It is easy to see that the data-partition alignment quality can be entirely different even for the same partition.

\begin{figure}[tp]
    \centering
    \includegraphics[width=0.57\linewidth]{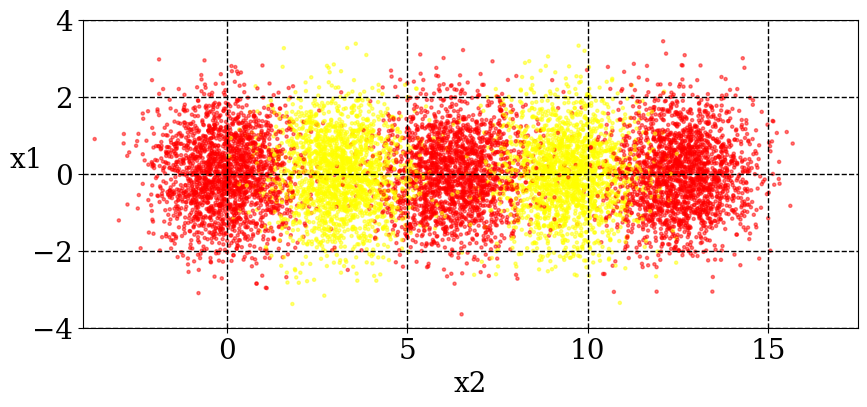} \\
    (a) Partition \textbf{T1}
    
    \includegraphics[width=0.6\linewidth]{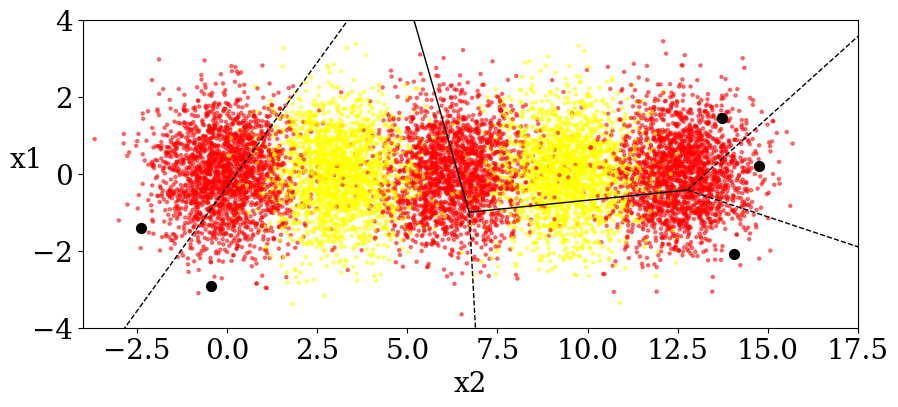} \\
    {(b) Partition \textbf{T2}}
    
    \includegraphics[width=0.6\linewidth]{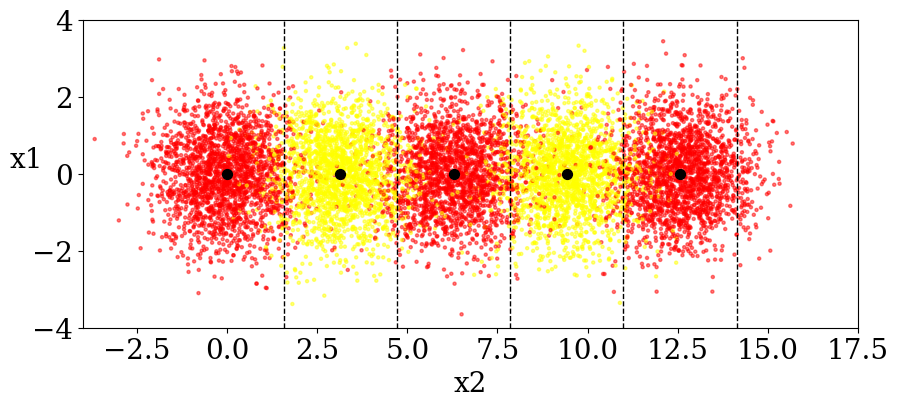} \\
    (c) Partition \textbf{T3}
    \caption{Alignment between data distribution and partition.}
    \label{fig-Alignment}
\end{figure}

\begin{figure}[tp]
    \centering
    \includegraphics[width=0.45\linewidth]{image/T3-partition.png}
    \includegraphics[width=0.45\linewidth]{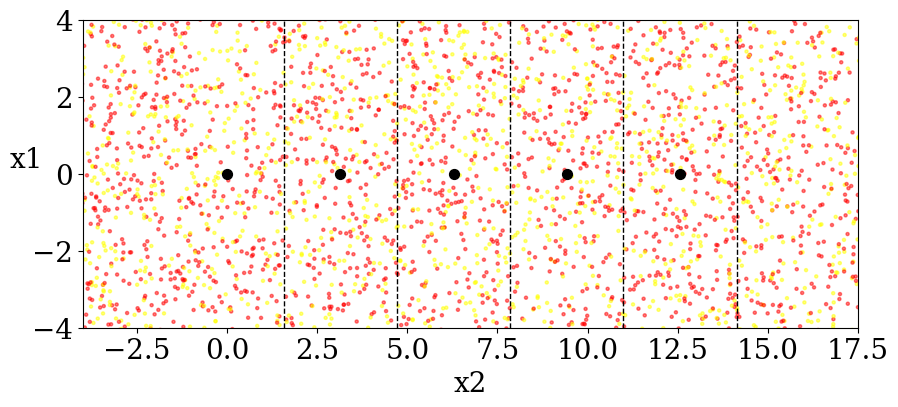} 
    \caption{The mixture model for the cases:  $\nu =1$ (left), and  $\nu =10^4$ (right).}
    \label{fig-Alignment-geometry}
\end{figure}

To evaluate the quality of these partitions, we generated $100000$ i.i.d. samples from the MM (with variance $\nu=1$) and computed $Unc(\Gamma)$ for varying $\alpha$, with $K = 100$, $\gamma = 0.04^{-1/\alpha}$ and $\delta = 0.01$. The results are reported in Table~\ref{tab:different-partitioning-strategies-geometries}(a). Among the three strategies, \textbf{T1} resulted in the highest uncertainty, while \textbf{T3} consistently produced the lowest. These findings suggest that partitions leading to balanced local measures (such as \textbf{T3}) are more favorable, while those poorly aligned with the data distribution (\textbf{T1}, \textbf{T2}) lead to higher uncertainty. This empirical evidence supports our theoretical discussion on the importance of selecting meaningful partitions.

\begin{table}[tp]
\centering
\caption{Uncertainty term $Unc(\Gamma)$ under (a) \textit{different partitioning strategies} (for fixed variance $\nu=1$) and (b) \textit{data geometries} (for fixed partition \textbf{T3}), as $\alpha$ changes. Smaller is better.}
\label{tab:different-partitioning-strategies-geometries}
\begin{minipage}{0.4\textwidth}
\centering
(a) \\
\small{
\begin{tabular}{|lcccc|}
\hline 
$\alpha$   & 4      & 6      & 8      & 10     \\ \hline
\textbf{T1} & 1.0045 & 0.8106 & 0.7315 & 0.6889 \\
\textbf{T2} & 0.8723 & 0.7302 & 0.6722 & 0.6409 \\
\textbf{T3} & 0.8450 & 0.7147 & 0.6615 & 0.6328 \\
\hline
\end{tabular}}
\end{minipage}%
\hfill
\begin{minipage}{0.50\textwidth}
\centering
(b) \\
\small{
\begin{tabular}{|lcccc|}
\hline 
$\alpha$ & 4      & 6      & 8      & 10     \\ \hline
$\nu=10^0$ & 0.8450 & 0.7147 & 0.6615 & 0.6328 \\
$\nu=10^2$ & 0.8458 & 0.7151 & 0.6618 & 0.6331 \\
$\nu=10^4$ & 1.0142 & 0.8393 & 0.7680 & 0.7295 \\
\hline
\end{tabular}}
\end{minipage}
\end{table}

\textbf{Exploring data geometries:} We further examine how the geometry of the data distribution influences the uncertainty term. To this end, we consider the same mixture model with varying variances $\nu \in \{10^{0}, 10^{2}, 10^{4}\}$ while fixing \textbf{T3} as the partition.  Note that increasing $\nu$ to $10^{4}$ significantly alters the geometry of the mixture model compared to the case $\nu = 1$. The corresponding uncertainty values are reported in Table~\ref{tab:different-partitioning-strategies-geometries}(b).
These results demonstrate that $Unc(\Gamma)$ can vary considerably depending on the geometry induced by the data distribution. When the partition $\Gamma$ does not align well with the data, the resulting local regions may have highly imbalanced probability measures. In such cases, the uncertainty can be large.

\subsection{Comparison with existing generalization bounds} \label{app-Comparation-prior-bounds}

To further clarify the advantages of our bound, we  carry out an additional comparison with robustness-based bounds developed by \citep{kawaguchi22RobustGen, than2025gentle}, which are also model-dependent. In this comparison, we apply our bound under the mild setting used in Table~\ref{tab:sup-imagenet-bound-train-only}, while we use $\delta = 0.05$ (corresponding to 95\% confidence) and utilize the ImageNet validation set to approximate the intractable components for the bounds in \citep{kawaguchi22RobustGen, than2025gentle}. 

The results across 17 pretrained models are summarized in Table~\ref{tab:compare-bounds}. The results suggest that our bound outperforms the existing robustness-based bounds in most cases, despite not relying on the validation set. This highlights the practical advantages and potential of our bound. 

\begin{table}[htp]
\centering
\caption{Comparison with different model-dependent bounds for pretrained models on ImageNet. The prior bounds are approximated from the validation set, due to their intractability.}
\label{tab:compare-bounds}
\small{
\begin{tabular}{|lcccc|}
\hline 
Model                & Test error & Bound (3) in  & Bound (8) in  & Our bound (\ref{thm-gen-train-small-K-any-distribution-eq})\\ 
 & & \citep{kawaguchi22RobustGen} & \citep{than2025gentle} & \\ \hline
ResNet18 V1          & 0.302      & 1.501            & 0.599                & 0.579         \\
ResNet34 V1          & 0.267      & 1.437            & 0.553                & 0.523         \\
ResNet50 V1          & 0.239      & 1.406            & 0.521                & 0.498         \\
ResNet101 V1         & 0.226      & 1.377            & 0.504                & 0.472         \\
ResNet152 V1         & 0.217      & 1.371            & 0.491                & 0.468         \\ 
SwinTransformer B    & 0.164      & 1.323            & 0.432                & 0.431         \\
SwinTransformer T    & 0.185      & 1.365            & 0.463                & 0.430         \\
SwinTransformer B V2 & 0.159      & 1.322            & 0.421                & 0.466         \\
SwinTransformer T V2 & 0.179      & 1.349            & 0.448                & 0.454         \\ 
VGG13                & 0.301      & 1.475            & 0.600                & 0.551         \\
VGG13 BN             & 0.284      & 1.478            & 0.580                & 0.559         \\
VGG19                & 0.276      & 1.444            & 0.565                & 0.528         \\
VGG19 BN             & 0.258      & 1.439            & 0.545                & 0.526         \\ 
DenseNet121          & 0.256      & 1.432            & 0.527                & 0.523         \\
DenseNet161          & 0.229      & 1.375            & 0.493                & 0.471         \\
DenseNet169          & 0.244      & 1.398            & 0.513                & 0.490         \\
DenseNet201          & 0.231      & 1.369            & 0.498                & 0.465         \\
\hline
\end{tabular}}
\end{table}

\subsection{Correlation between test error and our bounds}\label{app-Correlation-test-error-bounds}

We investigate how well our bounds can correlate with test error. Our bounds contain two main parts: (1) Training error and (2) Uncertainty term $Unc(\Gamma)$. Due to being simplified from Bound~(\ref{thm-gen-train-small-K-eq}), Bound~(\ref{thm-gen-train-small-K-any-distribution-eq})  may not exhibit the full strength of our bounds in this work. Therefore, we take Bound~(\ref{thm-gen-train-small-K-eq}) into consideration in this evaluation.

To examine the impact of the uncertainty term $Unc(\Gamma) = C\sqrt{\frac{u}{2n^2} \ln\frac{1}{\delta_1} } + g(\Gamma,\vh,\delta_2)$ in Bound~(\ref{thm-gen-train-small-K-eq}), we compute it from either the ImageNet training or validation set, using the mild setting for the paramteters.
The results are reported in Table~\ref{tab:correlation-to-test-error}.

The results suggest that the uncertainty term in (\ref{thm-gen-train-small-K-eq})  captures meaningful characteristics of the trained models and correlate strongly with the test error. It also contributes a great role to the bounds, since Bound (\ref{thm-gen-train-small-K-eq}) exhibits a near-perfect correlation to test error. These highlight the practical relevance of our bounds for performance estimation.

\begin{table}[tp]
\centering
\caption{Correlations to test error. $Unc(\Gamma)$ in Bound (\ref{thm-gen-train-small-K-eq}) is approximated from either the ImageNet \textit{training set} or  \textit{validation set}. Bound (\ref{thm-gen-train-small-K-any-distribution-eq}) is computed from the training set alone.}
\label{tab:correlation-to-test-error}
\begin{tabular}{|lc|}
\hline 
\textbf{Quantity}                  & \textbf{Correlation to test error} \\ \hline
Training error               & 0.7899                    \\
$Unc(\Gamma)$ (Train) & 0.7926                    \\
$Unc(\Gamma)$ (Valid) & 0.9918   \\
\hline
Bound (\ref{thm-gen-train-small-K-any-distribution-eq}) & 0.7899                    \\
Bound (\ref{thm-gen-train-small-K-eq}) & 0.9893   \\ \hline 
\end{tabular}
\end{table}

\subsection{Better understanding of generalization}\label{app-understanding-generalization}

We next investigate how can theoretical bounds reflect the performance of a trained model and how predictive are the bounds? This is important when one wants to understand the main factors that lead to better performance/generalization of a model. It is also important to compare two specific/trained models of interest. 

To this end, we take Bound (\ref{thm-gen-train-small-K-eq}) into consideration. Specifically, we focus on the following quantities for each model $h$ in the bound:
\begin{eqnarray}
    \mathrm{Align}(h) &=& \sum_{i \in T} a_i(h)\sqrt{n_i/n} \\
    \mathrm{Fair}(h) &=& \sum_{i \in T} a_i(h) \\
    \mathrm{Behavior}(h) &=& \mathrm{Align}(h)\cdot \sqrt{\frac{2 \ln(2K/\delta_2)}{n}}
    + \mathrm{Fair}(h)\cdot \frac{2\ln(2K/\delta_2)}{n}
\end{eqnarray}

Note that $Align(h)$ tells how well the model's local error can match with the data distribution. A better model should align better with the distribution's complexity, hence making $Align$ smaller. Meanwhile $Fair(h)$ tells the macro-level error of model $h$. It also suggests how fair for different local areas the model is. Finally, $Behavior(h)$ is the combined behavior, being an important part of Bound (\ref{thm-gen-train-small-K-eq}).

We use $K=200, \delta_2=0.01$ and the ImageNet validation set to compute those quantities. The results for 6 pretrained models are reported in Table~\ref{tab:understanding-generalization-Align-Fair}. We can observe that all of those quantities have extremely high correlations to the test error. $Align$ has the highest correlation, but $Fair$ has the lowest one. These results demonstrate that $Align$ can exhibit the quality of a model, and can be an accurate indicator for comparison between two models.

When visualizing the two lists $\{a_i: {i \in T}\}$ and $\{\sqrt{n_i/n}: {i \in T}\}$ in Figure~\ref{fig:local-error-to-distribution-complexity}, we observe a good correlation for some models, e.g., SwinTransformer B V2. Meanwhile ResNet18 V1 exhibited a much worse correlation. For areas with high probability mass (meaning large $\sqrt{n_i/n}$), those models often have small errors. However, those models have large errors on areas with very low probability density. The behavior in those areas (with probability mass $<0.05$) seems quite noisy. 

We further visualize those quantities $a_i$ and $\sqrt{n_i/n}$ for each local region in Figure~\ref{fig:local-error-and-distribution-complexity}. This visualization provides more details about the local behavior of a model, and supports further comparison between two models. For instance, while having comparable test error, ResNet152 V2 seems to slightly worse align with the distribution's complexity than VIT, specially for areas with low probability density. 

\begin{figure}[tp]
    \centering 
    \includegraphics[width=0.75\textwidth]{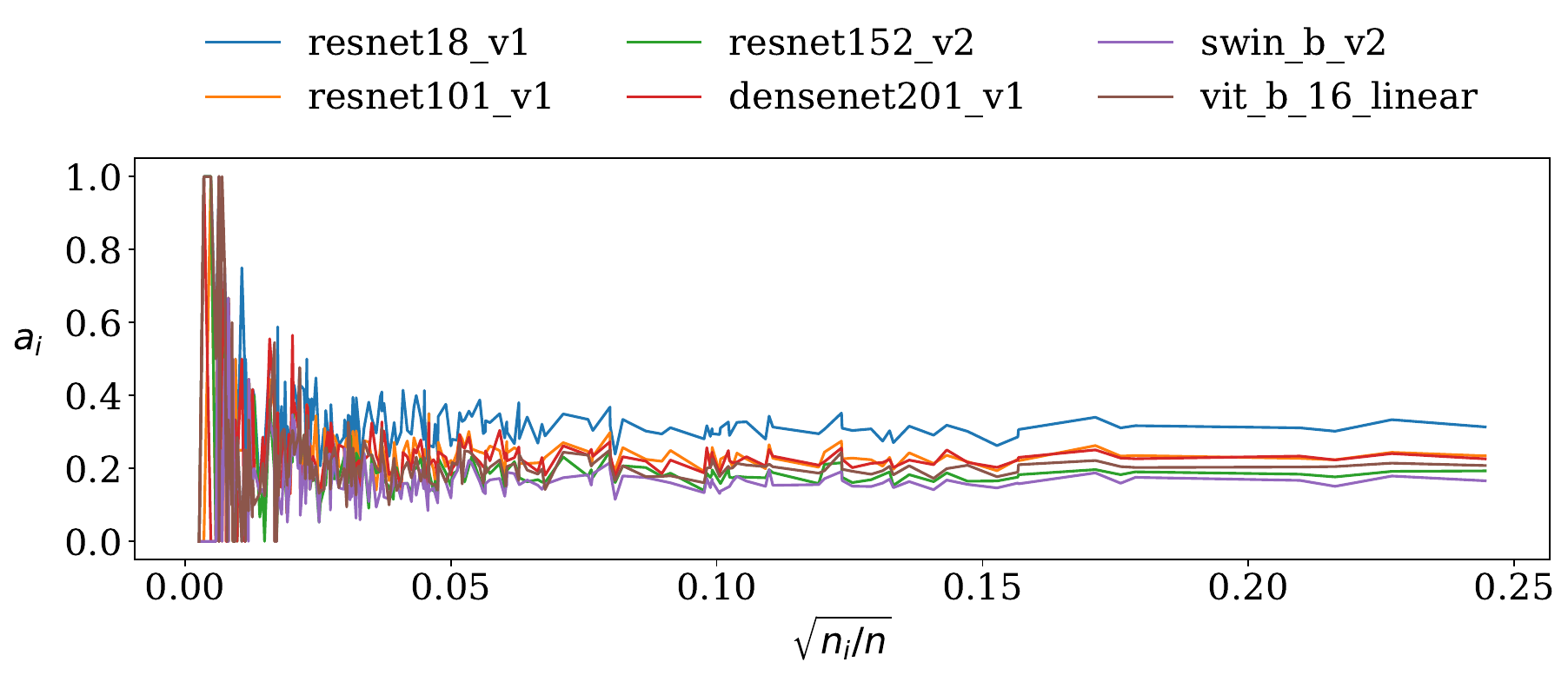}
    \caption{The alignment between local error with distribution's complexity. Note that quantities $\{\sqrt{n_i/n}: {i \in T}\}$ partly reflect the complexity of the data distribution. A high $\sqrt{n_i/n}$ means a high probability density in area $\gZ_i$.}
    \label{fig:local-error-to-distribution-complexity}
\end{figure}
\begin{figure}[tp]
    \centering
    \includegraphics[width=.75\textwidth]{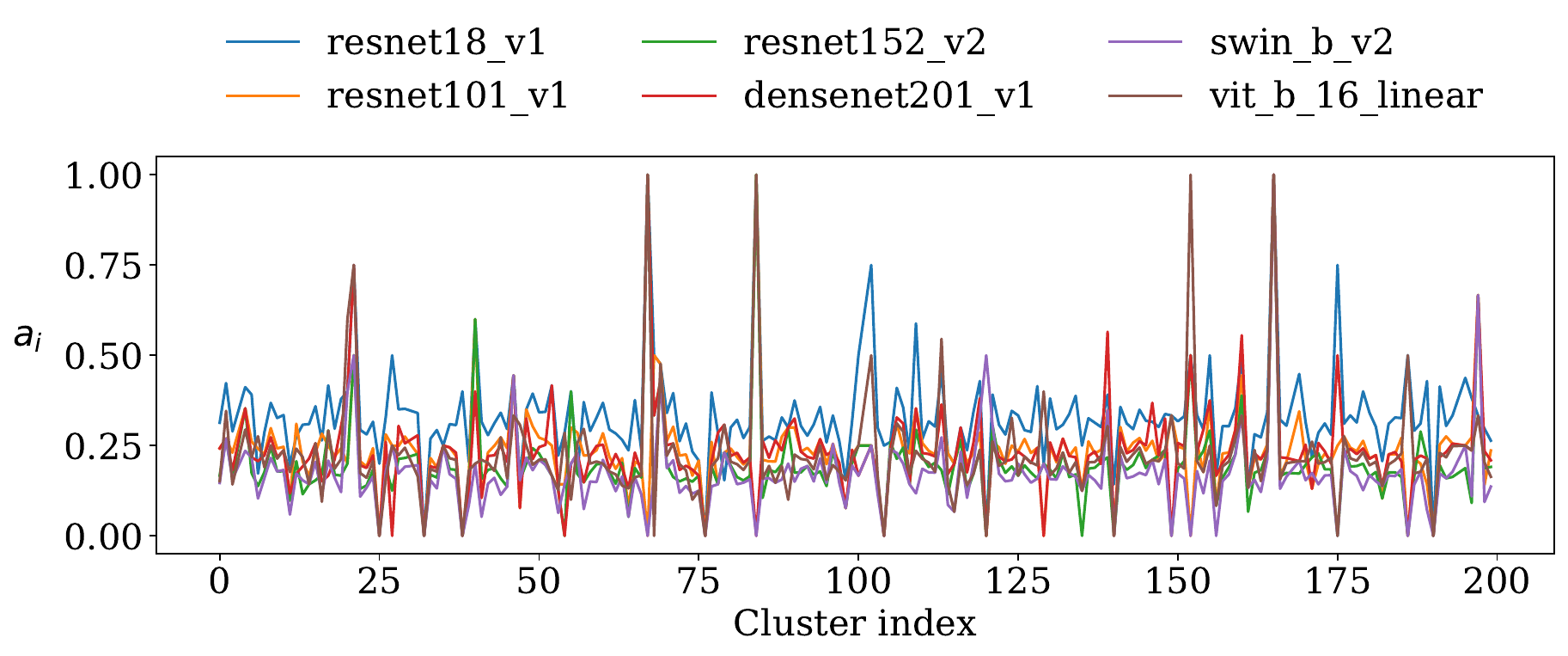}
    \includegraphics[width=.75\textwidth]{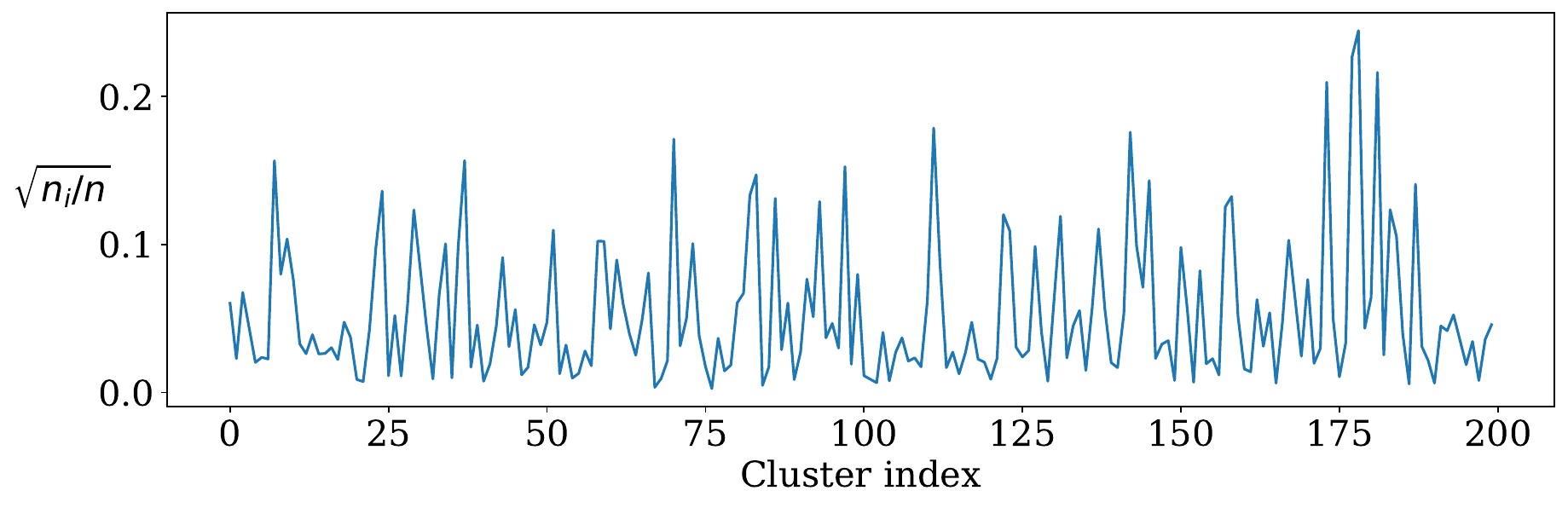}
    \caption{Distribution of local errors and samples on all local areas of the data space. Each cluster index represents an area.}
    \label{fig:local-error-and-distribution-complexity}
\end{figure}

\begin{remark}
The key insight from the result of this evaluation is that \textit{the nature of the strong correlation in Table~\ref{tab:understanding-generalization-Align-Fair} reveals which geometric aspects of the data distribution and localized model behavior drive the generalization gap}. In particular:
\begin{itemize}
    \item The decomposition in Bound (\ref{thm-gen-train-small-K-eq}) distinguishes \textit{distributional concentration} from \textit{local model stability}, two factors that classical norm-based or PAC-Bayes bounds cannot separate.
    \item The strong correlation between \textit{the data-model behavior alignment} (as measured by $Align$) and the test error suggests that such an alignment may be a critical factor in generalization for modern large models.
    \item The local-loss terms indicate where the model allocates most of its residual error mass, providing a structured way to diagnose local weaknesses (e.g., regions with inconsistent predictions).
\end{itemize}

These observations open concrete avenues for future algorithms, e.g., geometry-aware sampling, region-wise curriculum learning, adaptive partition refinement, or regularizers designed to smooth local variations. 
\end{remark}

\subsection{Computed bounds for smaller or simplier models}\label{app-simple-models}

In this section, we compute our bounds on several simple machine learning models for two tasks: classification and regression. Specifically, we train
\begin{itemize}
    \item \textit{three classifiers:} logistic regression, SVM, and XGBoost on a topic classification task using a news popularity dataset with 100K samples available at \url{https://archive.ics.uci.edu/dataset/432/news+popularity+in+multiple+social+media+platforms}. 
    \item \textit{three regression models:} linear regression, support vector regression (SVR), and XGBoost, using the VirusShare dataset  with 107K samples available at \url{https://archive.ics.uci.edu/dataset/413/dynamic+features+of+virusshare+executables}.
\end{itemize}

For all experiments, we split each dataset into training set and testing set with ratio 8:2 and compute our bound using the setting $K = 200$, $\delta_2 = 0.01$, $\delta_1 = 0.04$, $\alpha = 100$, as in other experiments. For the classification tasks, we employ the 0-1 loss to measure both training and test error, building the partition by clustering on the TFIDF vector space with random centroids initialized. For the regression tasks, we use the L1 loss for computing the training and test loss, building the partition by clustering on the standard-scaled vector space with random centroids initialized.

The evaluation results are reported in Table \ref{tab:simple-models}. We observe that our bound is non-vacuous for all the classifiers. For this experiments, our bound seem to be high and far from the test error. The main reason may come from the small size of the datasets (about 100K samples vs. 1.2M ImageNet images). Some other reasons may be the large value of $K=200$ and the misalignment between the partition and the data distribution.

\begin{table}[tp]
\centering
\caption{Results on Classification and Regression tasks.}
(a) Classification results.

\begin{tabular}{|l|c|c|c|}
\hline
\textbf{Model}      & \textbf{Test error} & \textbf{Train error} & \textbf{Our bound} \\ \hline
Logistic Regression & 0.02467             & 0.01905              & 0.81254        \\
SVM                 & 0.02284             & 0.01271              & 0.78078        \\
XGBoost                 & 0.02638             & 0.01750              & 0.78557        \\ \hline
\end{tabular}

(b) Regression results.

\begin{tabular}{|l|c|c|c|}
\hline
\textbf{Model}      & \textbf{Test error} & \textbf{Train error} & \textbf{Our bound} \\ \hline
Linear Regression &  0.14350            & 0.14187              & 1.06500        \\
SVR                 & 0.35403          & 0.35405
            & 1.27718       \\
XGBoost                & 0.10975             & 0.11070             & 1.03383        \\ \hline
\end{tabular}
\label{tab:simple-models}
\end{table}

\section{Computational complexity for computing certificates}\label{app-Computational-complexity}

We provide some analyses about the cost to compute our bound (\ref{thm-gen-train-small-K-any-distribution-eq}). This can help readers to see how cheap it is when compare with the existing bounds.

\textbf{Computational complexity.} Given a fixed partition $\Gamma=\{\gZ_i\}_{i=1}^K$, computing the bound in Theorem~\ref{thm-gen-train-small-K-any-distribution} requires a single pass over the dataset $\mS$. For each sample $\vz \in \mS$, we determine the region index $i$ such that $\vz \in \gZ_i$, increment the corresponding count $n_i$, and accumulate the empirical loss $F(\mS,\vh)$. This step takes $O(n \cdot T_\Gamma)$ time, where $T_\Gamma$ denotes the cost of identifying the region index for a given $\vz$ (e.g., constant time for simple partitions, or higher depending on the structure of $\Gamma$). Once the counts $\{n_i\}_{i=1}^K$ are obtained, all remaining quantities ($\hat{u}$, $g_2(\delta)$, and the bound) can be computed in $O(K)$ time. Therefore, the overall complexity is $O(n \cdot T_\Gamma + K)$ time and $O(K)$ memory.

\textbf{Construction cost of $\Gamma$.} The cost of constructing the partition $\Gamma$ is separate from evaluating the bound and depends on how $\Gamma$ is defined. For simple predefined partitions (e.g., uniform grids, randomly generated ones), construction requires $O(K)$ time. If $\Gamma$ is built from auxiliary data of size $m$ (e.g., via $k$-means clustering), the cost is typically $O(m K d \cdot I)$, where $d$ is the feature dimension and $I$ is the number of iterations. For representation-aware partitions using a pretrained model, one must additionally account for feature extraction, which costs $O(m \cdot T_{\mathrm{feat}})$, where $T_{\mathrm{feat}}$ is the cost of a forward pass. In many applications, $\Gamma$ can be constructed offline and reused across models, so this cost can be amortized; the per-model certification cost then remains $O(n \cdot T_\Gamma + K)$.

\textbf{Experiment.} We report some statistics about running time of different steps when computing our bound on ImageNet with 1.2M training samples, and 50K validation samples. K-means was used to define partition $\Gamma$, using the FAISS library. It was used for both input space and feature space induced by a pretrained ResNet-18. The results appear in Table~\ref{tab-sup-compute-cost}. They suggest that computing the training loss of a big model only can be very expensive, which is significantly more costly than running K-means  and assigning training samples to appropriate local areas.

\begin{table*}[tp]
\centering
\caption{Computational cost to do different steps to compute bound (\ref{thm-gen-train-small-K-any-distribution-eq}). All were done using 1 GPU P100.}
\label{tab-sup-compute-cost} 
\begin{tabular}{|lr|}
\hline 
 \textbf{Step}       &      \textbf{Time}  \\ \hline
 K-means to find centroids on Validation set (resized to 32x32x3) &  < 1' \\
 Assign training samples to local regions on input space & 2h10'57'' \\
 Assign training samples to local regions on Feature space of ResNet-18 & 3h25'04'' \\
 Compute the uncertainty term & < 1' \\
 \textbf{Compute training loss:} & \\
 	\hspace{15pt} RegNet Y 128GF e2e & ~58h \\
  	\hspace{15pt} VIT H 14 linear & ~20h \\ \hline
\end{tabular}
\end{table*}

\section{Related work}\label{sec-related-work}

Various approaches have been studied to analyze generalization capability. Those approaches connect different aspects of a learning algorithm or hypothesis to generalization.

\textbf{Norm-based bounds}  \citep{bartlett2017SpectralMarginDNN,golowich2020RC,galanti2023normDNN,graf2022measuring} is one of the earliest approaches to understand NNs. The existing studies often use Rademacher complexity to provide data- and model-dependent bounds on the generalization error. An NN with smaller weight norms will have a smaller bound, suggesting better generalization on unseen data. Nonetheless, the norms of weight matrices are often large for practical NNs \citep{arora2018strongerBounds}. Therefore, most existing norm-based bounds are vacuous.

\textbf{Algorithmic stability} \citep{bousquet2002stability,shalev2010StabilityLearnability,charles2018stability,kuzborskij2018stabilitySGD} is an approach to studying a learning algorithm. Basically, those theories suggest that a more stable algorithm can generalize better. Stable algorithms are less likely to overfit the training set, leading to more reliable predictions.  The stability requirement in those theories is that a replacement of one sample for the training set will not significantly change the  loss of the trained model. Such an assumption is really strong. One drawback is that achieving stability often requires restricting model complexity, potentially sacrificing predictive accuracy on challenging datasets. Therefore, this approach has a limited success in understanding deep NNs.

\textbf{Algorithmic robustness} \citep{xu2012robustnessGeneralize,sokolic2017robustDNN,kawaguchi22RobustGen,than2025gentle} is a framework to study generalization capability. It says that a robust learning algorithm can produce robust models which can generalize well on unseen data. This approach provides another lens to understand a learning algorithm and a trained model. However, it requires the assumption that the learning algorithm is robust, i.e., the loss of the trained model changes little in the small areas around the training samples. Such an assumption is really strong and cannot apply well for modern NNs, since many practical NNs suffer from adversarial attacks \citep{madry2018AdversarialTraining,zhou2022adversarial}. \citet{than2025gentle}  showed that those theories are often vacuous.

\textbf{Neural Tangent Kernel}   \citep{jacot2018NTK,arora2019fine}  provides a theoretical lens to study generalization of NNs by linking them to kernel methods in the infinite-width limit. As networks grow wider, their training dynamics under gradient descent can be approximated by a kernel function which remains constant throughout training. This perspective simplifies the analysis of complex neural architectures. The framework enables explicit generalization bounds, and a deeper understanding of how NN architecture and initialization affect learning. However, the main limitation of this framework  comes from its assumptions, such as the \textit{infinite-width} regime and fixed kernel during training, may not fully capture the behavior of finite, practical NNs. Some other studies \citep{lee2022NTK} can remove the infinite-width regime but assume the \textit{infinite depth}.

\textbf{Mutual information (MI)} \citep{xu2017information,feldman2019stability,nadjahi24slicingMI}
has emerged as a powerful tool for analyzing generalization by quantifying the dependency between a model's learned representations and the data. Since a trained model contains the (compressed) knowledge learned from the training samples, MI offers a principled framework for studying the trade-off between compression and predictive accuracy. However, the existing MI-based theories \citep{xu2017information,wang2021optimizingMI,sefidgaran2022rateMI,nadjahi24slicingMI} have a notable drawback: computing MI in high-dimensional, non-linear settings is computationally challenging. This drawback poses significant challenges for analyzing deep NNs, although \citep{nadjahi24slicingMI,dong2025exactly} obtained some promissing results on small NNs. 

\textbf{PAC-Bayes}   \citep{mcallester1999PACBayes,haddouche2023pac,biggs2023tighterPAC,awasthi2020pac,perez2021tighter} recently has received a great attention, and provide  non-vacuous bounds \citep{zhou2019CompressionBound,mustafa2024StochasticNN} for some NNs. Those bounds  often estimate $\E_{\hat{\vh}}  [F(P, \hat{\vh})]$ which is the expectation of the test error over the posterior distribution of $\hat{\vh}$. It means that those bounds are for a \textit{stochastic model} $\hat{\vh}$. Hence they provide limited understanding for a specific deterministic model $\vh$. \citet{neyshabur2018SpectralMarginDNN} provided an attempt to derandomization for PAC-Bayes but resulted in vacuous bounds for modern NNs \citep{arora2018strongerBounds}. Some recent attempts to derandomization include \citep{viallard2024PacBayes,clerico2025deterministicPAC}.

\textbf{Non-vacuous bounds for NNs:} \citet{dziugaite2017computingBounds} obtained a non-vacuous bound for NNs by finding a posterior distribution over neural network parameters that minimizes the PAC-Bayes bound. Their optimized bound is non-vacuous for a stochastic MLP with 3 layers trained on MNIST dataset. \citet{zhou2019CompressionBound} bounded the population loss of a stochastic NNs by using compressibility level of a NN. Using off-the-shelf neural network compression schemes, they provided the first non-vacuous bound for LeNet-5 and MobileNet, trained on ImageNet with more than 1.2M samples. \citet{lotfi2022pac} developed a compression method to further optimize the PAC-Bayes bound, and estimated the error rate of 40.9\% for MobileViT on ImageNet. \citet{mustafa2024StochasticNN}  provided a non-vacuous PAC-Bayes bound for adversarial population loss for VGG  on CIFAR10 dataset. \citet{galanti2023comparative} presented a PAC-Bayes bound which is non-vacuous for Convolutional NNs with up to 20 layers and for CIFAR10 and MNIST. \citet{akinwande24understanding} provided a non-vacuous PAC-Bayes bound for prompts. Although making a significant progress for NNs, those bounds are  non-vacuous for stochastic   neural networks only. \citet{biggs2022nonvacuousBound}  provided PAC-Bayes bounds for deterministic models and obtain (empirically) non-vacuous bounds for a specific class of (SHEL) NNs with a single hidden layer, trained on MNIST and Fashion-MNIST. Nonetheless, it is unclear about how well those bounds apply to bigger or deeper NNs. 

Towards understanding big/huge NNs, \citet{lotfi24NonVacuousLLM,lotfi24unlockingLLM} made a significant step that provides non-vacuous bounds for LLMs. While the PAC-Bayes bound in \citep{lotfi24NonVacuousLLM} can work  with LLMs trained from  i.i.d data, the recent bound in \citep{lotfi24unlockingLLM} considers token-level loss for LLMs and applies to dependent settings, which is close to the practice of training LLMs. Using both model quantization, finetuning and some other techniques, the PAC-Bayes bound by \citep{lotfi24unlockingLLM} is shown to be non-vacuous for huge LLMs, e.g., LLamMA2. Those bounds significantly push the frontier of learning theory towards building a solid foundation for DL. 

Nonetheless, there are two main drawbacks of those bounds \citep{lotfi24NonVacuousLLM,lotfi24unlockingLLM}. First, model quantization or compression is required in order to obtain a good bound. It means, those bounds are for the quantized or compressed models, and \textit{hence may not necessarily be true for the original (unquantized or uncompressed) models}. For example, \citep{lotfi24unlockingLLM} provided a non-vacuous bound for the 2-bit quantized versions of LLamMA2, instead of their original pretrained versions. Second, those bounds require the assumption that \textit{the model (hypothesis) family is finite}, meaning that a learning algorithm only searches in a space with finite number of specific models. Although such an assumption is reasonable for the current computer architectures, those bounds cannot explain a trained model that belongs to families with  infinite (or uncountable) number of members, which are provably prevalent. In contrast, our bounds apply directly to any specific model without requiring any modification or support. A comparison between our bounds and prior approaches about some key aspects is presented in Table~\ref{tab-Recent-approaches}.

%% file: Nonvacuous-DNN.bbl
\begin{thebibliography}{55}
\providecommand{\natexlab}[1]{#1}
\providecommand{\url}[1]{\texttt{#1}}
\expandafter\ifx\csname urlstyle\endcsname\relax
  \providecommand{\doi}[1]{doi: #1}\else
  \providecommand{\doi}{doi: \begingroup \urlstyle{rm}\Url}\fi

\bibitem[Achiam et~al.(2023)Achiam, Adler, Agarwal, Ahmad, Akkaya, Aleman,
  Almeida, Altenschmidt, Altman, Anadkat, et~al.]{achiam2023GPT4}
J.~Achiam, S.~Adler, S.~Agarwal, L.~Ahmad, I.~Akkaya, F.~L. Aleman, D.~Almeida,
  J.~Altenschmidt, S.~Altman, S.~Anadkat, et~al.
\newblock Gpt-4 technical report.
\newblock \emph{arXiv preprint arXiv:2303.08774}, 2023.

\bibitem[Akinwande et~al.(2024)Akinwande, Jiang, Sam, and
  Kolter]{akinwande24understanding}
V.~Akinwande, Y.~Jiang, D.~Sam, and J.~Z. Kolter.
\newblock Understanding prompt engineering may not require rethinking
  generalization.
\newblock In \emph{International Conference on Learning Representations}, 2024.

\bibitem[Arora et~al.(2018)Arora, Ge, Neyshabur, and
  Zhang]{arora2018strongerBounds}
S.~Arora, R.~Ge, B.~Neyshabur, and Y.~Zhang.
\newblock Stronger generalization bounds for deep nets via a compression
  approach.
\newblock In \emph{International Conference on Machine Learning}, pages
  254--263. PMLR, 2018.

\bibitem[Arora et~al.(2019)Arora, Du, Hu, Li, and Wang]{arora2019fine}
S.~Arora, S.~Du, W.~Hu, Z.~Li, and R.~Wang.
\newblock Fine-grained analysis of optimization and generalization for
  overparameterized two-layer neural networks.
\newblock In \emph{International Conference on Machine Learning}, pages
  322--332. PMLR, 2019.

\bibitem[Awasthi et~al.(2020)Awasthi, Kale, Karp, and Mohri]{awasthi2020pac}
P.~Awasthi, S.~Kale, S.~Karp, and M.~Mohri.
\newblock Pac-bayes learning bounds for sample-dependent priors.
\newblock \emph{Advances in Neural Information Processing Systems},
  33:\penalty0 4403--4414, 2020.

\bibitem[Bartlett et~al.(2017)Bartlett, Foster, and
  Telgarsky]{bartlett2017SpectralMarginDNN}
P.~L. Bartlett, D.~J. Foster, and M.~J. Telgarsky.
\newblock Spectrally-normalized margin bounds for neural networks.
\newblock \emph{Advances in Neural Information Processing Systems},
  30:\penalty0 6240--6249, 2017.

\bibitem[Biggs and Guedj(2022)]{biggs2022nonvacuousBound}
F.~Biggs and B.~Guedj.
\newblock Non-vacuous generalisation bounds for shallow neural networks.
\newblock In \emph{International Conference on Machine Learning}, pages
  1963--1981. PMLR, 2022.

\bibitem[Biggs and Guedj(2023)]{biggs2023tighterPAC}
F.~Biggs and B.~Guedj.
\newblock Tighter pac-bayes generalisation bounds by leveraging example
  difficulty.
\newblock In \emph{International Conference on Artificial Intelligence and
  Statistics}, pages 8165--8182. PMLR, 2023.

\bibitem[Bousquet and Elisseeff(2002)]{bousquet2002stability}
O.~Bousquet and A.~Elisseeff.
\newblock Stability and generalization.
\newblock \emph{The Journal of Machine Learning Research}, 2:\penalty0
  499--526, 2002.

\bibitem[Brutzkus and Globerson(2021)]{brutzkus2021optimization}
A.~Brutzkus and A.~Globerson.
\newblock An optimization and generalization analysis for max-pooling networks.
\newblock In \emph{Uncertainty in Artificial Intelligence}, pages 1650--1660.
  PMLR, 2021.

\bibitem[Charles and Papailiopoulos(2018)]{charles2018stability}
Z.~Charles and D.~Papailiopoulos.
\newblock Stability and generalization of learning algorithms that converge to
  global optima.
\newblock In \emph{International Conference on Machine Learning}, pages
  745--754. PMLR, 2018.

\bibitem[Clerico et~al.(2025)Clerico, Farghly, Deligiannidis, Guedj, and
  Doucet]{clerico2025deterministicPAC}
E.~Clerico, T.~Farghly, G.~Deligiannidis, B.~Guedj, and A.~Doucet.
\newblock Generalisation under gradient descent via deterministic pac-bayes.
\newblock In \emph{International Conference on Algorithmic Learning Theory},
  2025.

\bibitem[Dong et~al.(2025)Dong, Guo, Gong, Wen, and Li]{dong2025exactly}
Y.~Dong, H.~Guo, T.~Gong, W.~Wen, and C.~Li.
\newblock Exactly tight information-theoretic generalization bounds via binary
  jensen-shannon divergence.
\newblock In \emph{Forty-second International Conference on Machine Learning},
  2025.

\bibitem[Dubhashi and Ranjan(1998)]{dubhashi1998balls}
D.~P. Dubhashi and D.~Ranjan.
\newblock Balls and bins: A study in negative dependence.
\newblock \emph{Random Structures \& Algorithms}, 13\penalty0 (2):\penalty0
  99--124, 1998.

\bibitem[Dziugaite and Roy(2017)]{dziugaite2017computingBounds}
G.~K. Dziugaite and D.~M. Roy.
\newblock Computing nonvacuous generalization bounds for deep (stochastic)
  neural networks with many more parameters than training data.
\newblock In \emph{Conference on Uncertainty in Artificial Intelligence (UAI)},
  2017.

\bibitem[Feldman and Vondrak(2019)]{feldman2019stability}
V.~Feldman and J.~Vondrak.
\newblock High probability generalization bounds for uniformly stable
  algorithms with nearly optimal rate.
\newblock In \emph{Conference on Learning Theory (COLT)}, pages 1270--1279.
  PMLR, 2019.

\bibitem[Figalli(2018)]{figalli2018continuity}
A.~Figalli.
\newblock On the continuity of center-outward distribution and quantile
  functions.
\newblock \emph{Nonlinear Analysis}, 177:\penalty0 413--421, 2018.

\bibitem[Galanti et~al.(2023{\natexlab{a}})Galanti, Galanti, and
  Ben-Shaul]{galanti2023comparative}
T.~Galanti, L.~Galanti, and I.~Ben-Shaul.
\newblock Comparative generalization bounds for deep neural networks.
\newblock \emph{Transactions on Machine Learning Research}, 2023{\natexlab{a}}.

\bibitem[Galanti et~al.(2023{\natexlab{b}})Galanti, Xu, Galanti, and
  Poggio]{galanti2023normDNN}
T.~Galanti, M.~Xu, L.~Galanti, and T.~Poggio.
\newblock Norm-based generalization bounds for sparse neural networks.
\newblock \emph{Advances in Neural Information Processing Systems}, 36,
  2023{\natexlab{b}}.

\bibitem[Golowich et~al.(2020)Golowich, Rakhlin, and Shamir]{golowich2020RC}
N.~Golowich, A.~Rakhlin, and O.~Shamir.
\newblock Size-independent sample complexity of neural networks.
\newblock \emph{Information and Inference: A Journal of the IMA}, 9\penalty0
  (2):\penalty0 473--504, 2020.

\bibitem[Graf et~al.(2022)Graf, Zeng, Rieck, Niethammer, and
  Kwitt]{graf2022measuring}
F.~Graf, S.~Zeng, B.~Rieck, M.~Niethammer, and R.~Kwitt.
\newblock On measuring excess capacity in neural networks.
\newblock \emph{Advances in Neural Information Processing Systems},
  35:\penalty0 10164--10178, 2022.

\bibitem[Haddouche and Guedj(2023)]{haddouche2023pac}
M.~Haddouche and B.~Guedj.
\newblock Pac-bayes generalisation bounds for heavy-tailed losses through
  supermartingales.
\newblock \emph{Transactions on Machine Learning Research}, 2023.

\bibitem[Hallin et~al.(2021)Hallin, del Barrio, Cuesta-Albertos, and
  Matr{\'a}n]{hallin2021distribution}
M.~Hallin, E.~del Barrio, J.~Cuesta-Albertos, and C.~Matr{\'a}n.
\newblock Distribution and quantile functions, ranks and signs in dimension d:
  A measure transportation approach.
\newblock \emph{Annals of Statistics}, 49\penalty0 (2):\penalty0 1139--1165,
  2021.

\bibitem[Hou et~al.(2023)Hou, Kassraie, Kratsios, Krause, and
  Rothfuss]{hou2023instanceGen}
S.~Hou, P.~Kassraie, A.~Kratsios, A.~Krause, and J.~Rothfuss.
\newblock Instance-dependent generalization bounds via optimal transport.
\newblock \emph{Journal of Machine Learning Research}, 24:\penalty0 1--50,
  2023.

\bibitem[Jacot et~al.(2018)Jacot, Gabriel, and Hongler]{jacot2018NTK}
A.~Jacot, F.~Gabriel, and C.~Hongler.
\newblock Neural tangent kernel: convergence and generalization in neural
  networks.
\newblock In \emph{Advances in Neural Information Processing Systems}, pages
  8580--8589, 2018.

\bibitem[Joag-Dev and Proschan(1983)]{joagdev1983negative}
K.~Joag-Dev and F.~Proschan.
\newblock Negative association of random variables with applications.
\newblock \emph{The Annals of Statistics}, 11\penalty0 (1):\penalty0 286--295,
  1983.
\newblock \doi{10.1214/aos/1176346060}.

\bibitem[Jumper et~al.(2021)Jumper, Evans, Pritzel, Green, Figurnov,
  Ronneberger, Tunyasuvunakool, Bates, {\v{Z}}{\'\i}dek, Potapenko,
  et~al.]{jumper2021AlphaFold}
J.~Jumper, R.~Evans, A.~Pritzel, T.~Green, M.~Figurnov, O.~Ronneberger,
  K.~Tunyasuvunakool, R.~Bates, A.~{\v{Z}}{\'\i}dek, A.~Potapenko, et~al.
\newblock Highly accurate protein structure prediction with alphafold.
\newblock \emph{Nature}, 596\penalty0 (7873):\penalty0 583--589, 2021.

\bibitem[Kawaguchi et~al.(2022)Kawaguchi, Deng, Luh, and
  Huang]{kawaguchi22RobustGen}
K.~Kawaguchi, Z.~Deng, K.~Luh, and J.~Huang.
\newblock Robustness implies generalization via data-dependent generalization
  bounds.
\newblock In \emph{International Conference on Machine Learning}, volume 162 of
  \emph{Proceedings of Machine Learning Research}, pages 10866--10894. PMLR,
  2022.

\bibitem[Kuzborskij and Lampert(2018)]{kuzborskij2018stabilitySGD}
I.~Kuzborskij and C.~Lampert.
\newblock Data-dependent stability of stochastic gradient descent.
\newblock In \emph{International Conference on Machine Learning}, pages
  2815--2824. PMLR, 2018.

\bibitem[Lee et~al.(2022)Lee, Choi, Ryu, and No]{lee2022NTK}
J.~Lee, J.~Y. Choi, E.~K. Ryu, and A.~No.
\newblock Neural tangent kernel analysis of deep narrow neural networks.
\newblock In \emph{International Conference on Machine Learning}, pages
  12282--12351, 2022.

\bibitem[Lei and Ying(2020)]{lei2020stabilitySGD}
Y.~Lei and Y.~Ying.
\newblock Fine-grained analysis of stability and generalization for stochastic
  gradient descent.
\newblock In \emph{International Conference on Machine Learning}, pages
  5809--5819. PMLR, 2020.

\bibitem[Li et~al.(2024)Li, Zhu, and Liu]{li24algorithmic}
S.~Li, B.~Zhu, and Y.~Liu.
\newblock Algorithmic stability unleashed: Generalization bounds with unbounded
  losses.
\newblock In \emph{International Conference on Machine Learning}, 2024.

\bibitem[Lotfi et~al.(2022)Lotfi, Finzi, Kapoor, Potapczynski, Goldblum, and
  Wilson]{lotfi2022pac}
S.~Lotfi, M.~Finzi, S.~Kapoor, A.~Potapczynski, M.~Goldblum, and A.~G. Wilson.
\newblock Pac-bayes compression bounds so tight that they can explain
  generalization.
\newblock In \emph{Advances in Neural Information Processing Systems},
  volume~35, pages 31459--31473, 2022.

\bibitem[Lotfi et~al.(2024{\natexlab{a}})Lotfi, Finzi, Kuang, Rudner, Goldblum,
  and Wilson]{lotfi24NonVacuousLLM}
S.~Lotfi, M.~A. Finzi, Y.~Kuang, T.~G. Rudner, M.~Goldblum, and A.~G. Wilson.
\newblock Non-vacuous generalization bounds for large language models.
\newblock In \emph{International Conference on Machine Learning},
  2024{\natexlab{a}}.

\bibitem[Lotfi et~al.(2024{\natexlab{b}})Lotfi, Kuang, Finzi, Amos, Goldblum,
  and Wilson]{lotfi24unlockingLLM}
S.~Lotfi, Y.~Kuang, M.~A. Finzi, B.~Amos, M.~Goldblum, and A.~G. Wilson.
\newblock Unlocking tokens as data points for generalization bounds on larger
  language models.
\newblock In \emph{Advances in Neural Information Processing Systems},
  2024{\natexlab{b}}.

\bibitem[Madry et~al.(2018)Madry, Makelov, Schmidt, Tsipras, and
  Vladu]{madry2018AdversarialTraining}
A.~Madry, A.~Makelov, L.~Schmidt, D.~Tsipras, and A.~Vladu.
\newblock Towards deep learning models resistant to adversarial attacks.
\newblock In \emph{International Conference on Learning Representations}, 2018.

\bibitem[McAllester(1999)]{mcallester1999PACBayes}
D.~A. McAllester.
\newblock Some pac-bayesian theorems.
\newblock \emph{Machine Learning}, 37\penalty0 (3):\penalty0 355--363, 1999.

\bibitem[McAllester(2003)]{mcallester2003pac}
D.~A. McAllester.
\newblock Pac-bayesian stochastic model selection.
\newblock \emph{Machine Learning}, 51\penalty0 (1):\penalty0 5--21, 2003.

\bibitem[Mohri et~al.(2018)Mohri, Rostamizadeh, and
  Talwalkar]{mohri2018foundations}
M.~Mohri, A.~Rostamizadeh, and A.~Talwalkar.
\newblock \emph{Foundations of Machine Learning}.
\newblock MIT Press, 2018.

\bibitem[Mustafa et~al.(2024)Mustafa, Liznerski, Ledent, Wagner, Wang, and
  Kloft]{mustafa2024StochasticNN}
W.~Mustafa, P.~Liznerski, A.~Ledent, D.~Wagner, P.~Wang, and M.~Kloft.
\newblock Non-vacuous generalization bounds for adversarial risk in stochastic
  neural networks.
\newblock In \emph{International Conference on Artificial Intelligence and
  Statistics}, pages 4528--4536, 2024.

\bibitem[Nadjahi et~al.(2024)Nadjahi, Greenewald, Gabrielsson, and
  Solomon]{nadjahi24slicingMI}
K.~Nadjahi, K.~Greenewald, R.~B. Gabrielsson, and J.~Solomon.
\newblock Slicing mutual information generalization bounds for neural networks.
\newblock In \emph{International Conference on Machine Learning}, 2024.

\bibitem[Neyshabur et~al.(2018)Neyshabur, Bhojanapalli, and
  Srebro]{neyshabur2018SpectralMarginDNN}
B.~Neyshabur, S.~Bhojanapalli, and N.~Srebro.
\newblock A pac-bayesian approach to spectrally-normalized margin bounds for
  neural networks.
\newblock In \emph{International Conference on Learning Representations}, 2018.

\bibitem[P{\'e}rez-Ortiz et~al.(2021)P{\'e}rez-Ortiz, Rivasplata, Shawe-Taylor,
  and Szepesv{\'a}ri]{perez2021tighter}
M.~P{\'e}rez-Ortiz, O.~Rivasplata, J.~Shawe-Taylor, and C.~Szepesv{\'a}ri.
\newblock Tighter risk certificates for neural networks.
\newblock \emph{Journal of Machine Learning Research}, 22\penalty0
  (227):\penalty0 1--40, 2021.

\bibitem[Sefidgaran et~al.(2022)Sefidgaran, Gohari, Richard, and
  Simsekli]{sefidgaran2022rateMI}
M.~Sefidgaran, A.~Gohari, G.~Richard, and U.~Simsekli.
\newblock Rate-distortion theoretic generalization bounds for stochastic
  learning algorithms.
\newblock In \emph{Conference on Learning Theory}, pages 4416--4463. PMLR,
  2022.

\bibitem[Shalev-Shwartz et~al.(2010)Shalev-Shwartz, Shamir, Srebro, and
  Sridharan]{shalev2010StabilityLearnability}
S.~Shalev-Shwartz, O.~Shamir, N.~Srebro, and K.~Sridharan.
\newblock Learnability, stability and uniform convergence.
\newblock \emph{The Journal of Machine Learning Research}, 11:\penalty0
  2635--2670, 2010.

\bibitem[Silver et~al.(2016)Silver, Huang, Maddison, Guez, Sifre, Van
  Den~Driessche, Schrittwieser, Antonoglou, Panneershelvam, Lanctot,
  et~al.]{silver2016mastering}
D.~Silver, A.~Huang, C.~J. Maddison, A.~Guez, L.~Sifre, G.~Van Den~Driessche,
  J.~Schrittwieser, I.~Antonoglou, V.~Panneershelvam, M.~Lanctot, et~al.
\newblock Mastering the game of go with deep neural networks and tree search.
\newblock \emph{Nature}, 529\penalty0 (7587):\penalty0 484--489, 2016.

\bibitem[Sokoli{\'c} et~al.(2017)Sokoli{\'c}, Giryes, Sapiro, and
  Rodrigues]{sokolic2017robustDNN}
J.~Sokoli{\'c}, R.~Giryes, G.~Sapiro, and M.~R. Rodrigues.
\newblock Robust large margin deep neural networks.
\newblock \emph{IEEE Transactions on Signal Processing}, 65\penalty0
  (16):\penalty0 4265--4280, 2017.

\bibitem[Than et~al.(2025)Than, Phan, and Vu]{than2025gentle}
K.~Than, D.~Phan, and G.~Vu.
\newblock Gentle local robustness implies generalization.
\newblock \emph{Machine Learning}, 114\penalty0 (6):\penalty0 142, 2025.

\bibitem[Viallard et~al.(2024)Viallard, Germain, Habrard, and
  Morvant]{viallard2024PacBayes}
P.~Viallard, P.~Germain, A.~Habrard, and E.~Morvant.
\newblock A general framework for the practical disintegration of pac-bayesian
  bounds.
\newblock \emph{Machine Learning}, 113\penalty0 (2):\penalty0 519--604, 2024.

\bibitem[von Luxburg and Bousquet(2004)]{Luxburg04Lipschitz}
U.~von Luxburg and O.~Bousquet.
\newblock Distance-based classification with lipschitz functions.
\newblock \emph{J. Mach. Learn. Res.}, 5\penalty0 (Jun):\penalty0 669--695,
  2004.

\bibitem[Wang et~al.(2021)Wang, Zhang, Zhang, Meng, Chen, and
  Liu]{wang2021optimizingMI}
B.~Wang, H.~Zhang, J.~Zhang, Q.~Meng, W.~Chen, and T.-Y. Liu.
\newblock Optimizing information-theoretical generalization bound via
  anisotropic noise of sgld.
\newblock In \emph{Advances in Neural Information Processing Systems},
  volume~34, pages 26080--26090, 2021.

\bibitem[Xu and Raginsky(2017)]{xu2017information}
A.~Xu and M.~Raginsky.
\newblock Information-theoretic analysis of generalization capability of
  learning algorithms.
\newblock In \emph{Advances in Neural Information Processing Systems},
  volume~30, 2017.

\bibitem[Xu and Mannor(2012)]{xu2012robustnessGeneralize}
H.~Xu and S.~Mannor.
\newblock Robustness and generalization.
\newblock \emph{Machine Learning}, 86\penalty0 (3):\penalty0 391--423, 2012.

\bibitem[Zhou et~al.(2022)Zhou, Liu, Ye, Zhu, Zhou, and
  Yu]{zhou2022adversarial}
S.~Zhou, C.~Liu, D.~Ye, T.~Zhu, W.~Zhou, and P.~S. Yu.
\newblock Adversarial attacks and defenses in deep learning: From a perspective
  of cybersecurity.
\newblock \emph{ACM Computing Surveys}, 55\penalty0 (8):\penalty0 1--39, 2022.

\bibitem[Zhou et~al.(2019)Zhou, Veitch, Austern, Adams, and
  Orbanz]{zhou2019CompressionBound}
W.~Zhou, V.~Veitch, M.~Austern, R.~P. Adams, and P.~Orbanz.
\newblock Non-vacuous generalization bounds at the imagenet scale: a
  pac-bayesian compression approach.
\newblock In \emph{International Conference on Learning Representations
  (ICLR)}, 2019.

\end{thebibliography}
